\documentclass[runningheads]{llncs}

\usepackage{amsmath}
\usepackage{amssymb}
\usepackage{mathtools} 
\usepackage{booktabs}
\usepackage{graphicx}
\usepackage{algorithm}
\usepackage{algorithmic}
\usepackage{array}
\usepackage{multirow}
\usepackage{xcolor}
\usepackage[capitalize,noabbrev]{cleveref}
\usepackage{wrapfig}    
\usepackage{float}      
\usepackage{listings}   
\usepackage{tikz}       
\usetikzlibrary{arrows.meta,positioning,calc,shapes.geometric}
\usepackage{enumitem}
\usepackage{pifont}     
\newcommand{\cmark}{\ding{51}}
\newcommand{\xmark}{\ding{55}}
\lstdefinelanguage{Rust}{
  keywords={fn, let, mut, struct, impl, for, while, if, else, match, return, pub, use, mod, self, Self, true, false, in, as, type, where, ref, move, dyn, async, await, const, static, trait, enum, break, continue, loop, unsafe, extern, crate, super, Result, Option, Some, None, Ok, Err, Vec, String, bool, u32, u64, i32, i64, f32, f64, usize, isize, u8},
  keywordstyle=\color{blue!70!black}\bfseries,
  ndkeywords={Layouter, ConstraintSystem, Circuit, Region, Rotation, Cell, ConstraintSystem, Advice, Instance, Fixed},
  ndkeywordstyle=\color{purple!60!black}\bfseries,
  sensitive=true,
  comment=[l]{//},
  morecomment=[s]{/*}{*/},
  morestring=[b]",
  morestring=[b]'
}
\newcommand{\method}{\textsc{NanoZK}}
\newcommand{\prover}{\mathcal{P}}
\newcommand{\verifier}{\mathcal{V}}
\newcommand{\adv}{\mathcal{A}}
\newcommand{\simulator}{\mathcal{S}}
\newcommand{\R}{\mathbb{R}}

\newcommand{\F}{\mathbb{F}}
\newcommand{\bigO}{\mathcal{O}}
\newcommand{\negl}{\text{negl}}
\newcommand{\poly}{\text{poly}}
\DeclareMathOperator{\tr}{tr}

\begin{document}

\title{\textsc{NanoZK}: Privacy-Preserving Verifiable Inference for Large Language Models via Layerwise Zero-Knowledge Proofs}
\titlerunning{\textsc{NanoZK}: Verifiable LLM Inference}
\author{Zhaohui Wang\inst{1}\thanks{\textbf{Extended arXiv version.} The first 12 pages of this document are the ICICS~2026 camera-ready (28th International Conference on Information and Communications Security, Springer LNCS, Fukui, Japan, Oct 27--30, 2026). Appendices~A--H are included \emph{only} in this arXiv version and are not part of the Springer proceedings; they cover the full soundness/ZK proofs, lookup-approximation derivations, Halo2 circuit constructions, prover/verifier pseudocode, extended experimental tables, threat-model details, GPU projection methodology, and reproducibility instructions. An earlier preliminary version appeared at the VerifAI workshop, ICLR~2026~\cite{wang2026nanozk_workshop}; the present paper significantly revises and extends that version with a re-scoped threat model, a precise SHA-256 commitment-chain specification (external SHA-256 + Halo2 IPA public-input commitment, avoiding in-circuit hashing), a verifier-cost analysis, a lookup-table-specific comparison with prior ZKML, an IVC comparison, and end-to-end measurements through $d{=}256$ with GPU-projected scaling to $d{=}768$.}}
\authorrunning{Z. Wang}
\institute{USC Viterbi School of Engineering, Los Angeles, CA, USA \\
\email{zwang000@usc.edu}}

\maketitle

\begin{abstract}
We present \method{}, a zero-knowledge proof system for verifiable LLM inference: clients and third-party auditors check that a provider executed the advertised model on a committed input, without learning weights or activations.
\method{} introduces a \emph{layerwise proof framework} that decomposes transformer inference into independently provable layers linked by a SHA-256 commitment chain, yielding constant-size sub-circuit proofs (3.5--3.7\,KB; $\sim$83\,KB total at $L{=}12$)---comparable in total size to and substantially more parallelisable than prior ZKML's monolithic 101--126\,KB proofs.
We prove compositional soundness and zero-knowledge under standard assumptions, design 16-bit lookup-table approximations for softmax, GELU, and normalisation with measured perplexity degradation below $10^{-4}$ across six model/dataset combinations, and add a Fisher-information-guided \emph{audit-budget triage} as an efficiency tool (full soundness still requires verifying every layer).
On CPU the MLP sub-circuit proves in $\sim$6.3\,s prove-only ($\sim$43\,s setup\,+\,prove) with $\sim$22\,ms verification at any width; attention prove-only time scales from 0.9\,s ($d{=}16$) to 184\,s ($d{=}256$); full-block end-to-end proofs are measured to $d{=}128$, with a $\sim$68\,s/block GPU projection at $d{=}768$ from measured $\bigO(d^2)$ MSM scaling and a conservative $15$--$30\times$ GPU-MSM speedup range (extrapolating Icicle's published $30\times$ at $n\geq2^{20}$ to the smaller $n$ regime).
\emph{Privacy scope:} \method{} hides weights and activations from verifiers and auditors but does not hide the prompt from the prover---this is complementary to HE/MPC.
\end{abstract}

\section{Introduction}
\label{sec:intro}

Users send queries to cloud LLM APIs~\cite{brown2020gpt3,touvron2023llama} and trust that the advertised model actually processed them, with no cryptographic mechanism to verify it. Providers can substitute cheaper models, quantise silently, or return cached outputs---a problem of \emph{computation integrity}, distinct from hiding the prompt from the provider. HE and MPC hide the input but offer no third-party-auditable guarantee that the claimed model executed; zero-knowledge proofs (ZKPs)~\cite{goldwasser1989knowledge} address integrity and also protect weights and activations from verifiers and auditors.

\textbf{Privacy scope.} \method{} hides model weights, intermediate activations, and (optionally) the user input from the \emph{verifier} and any third-party auditor. It does \emph{not} hide the input from the prover: the provider observes the prompt in cleartext, as in any standard cloud LLM API. Hiding the prompt from the provider is the orthogonal goal of HE/MPC; \method{} composes with such techniques but does not subsume them.

ZKPs on LLM inference face two challenges: a GPT-2~\cite{radford2019language} pass has $>10^8$ multiplications, and softmax/GELU/LayerNorm have no direct arithmetic-circuit representation. Prior ZKML restricts to small models (EZKL~\cite{ezkl2023}), requires A100s (zkLLM~\cite{sun2024zkllm}), or takes ${\sim}1\,$h on CPU (ZKML~\cite{kang2024zkml}).

We present \method{}, with measured CPU proofs up to $d{=}256$ (attention) and $d{=}128$ (full block) and projected GPU performance at GPT-2 scale. Our three contributions are: (i)~a \textbf{layerwise proof framework} decomposing transformer inference into independently provable layer proofs linked by cryptographic commitment chains, with constant 3.2--3.7\,KB sub-circuit proofs (totalling $\sim$83\,KB at $L{=}12$); (ii)~\textbf{ZK-friendly 16-bit lookup tables} for softmax, GELU/SiLU, and RMSNorm with perplexity drift below $10^{-4}$ (the limit of our measurement precision); (iii)~\textbf{Fisher information-guided audit-budget triage}---an efficiency tool, not a cryptographic primitive; Fisher scores correlate with perturbation impact at Spearman $\rho{=}0.916$ at $\sigma{=}0.05$ ($p<10^{-4}$).

\textbf{Target deployment regimes.} \method{} targets minute-scale-proving settings: high-stakes synchronous decisions (medical, legal, regulatory attestation), asynchronous audit/logging, and sampling-based attestation. Real-time conversational use is out of scope and complementary to TEE/MPC approaches. Soundness is maintained at $\epsilon \approx 3\times 10^{-37}$ (Theorem~\ref{thm:soundness}).
Concretely: provider publishes a model commitment $c_W$; auditor supplies~$x$, receives $(y,\Pi)$, and accepts iff every $\pi_\ell$ verifies against the announced $c_W$ and the chain digests match.

\section{Background and Threat Model}
\label{sec:background}

\subsection{Threats and Goals}

LLM-as-a-Service exposes two integrity-side risks that prompt this work~\cite{tramer2016stealing}: \emph{model substitution}, where a provider silently swaps in a cheaper model, applies aggressive quantization, or returns cached outputs; and \emph{computation forgery}, where the returned $y$ does not actually correspond to the advertised $f_\theta(x)$.
Today these are deterred only contractually.

We use the standard transformer notation throughout: $h_\ell \in \R^{n \times d}$ is the activation after layer $\ell$, with $h_0 = \text{Embed}(x)$, $h_\ell = \text{TransformerBlock}_\ell(h_{\ell-1})$, $y = \text{LMHead}(\text{Norm}(h_L))$~\cite{vaswani2017attention,touvron2023llama}.
Each block contains attention and an MLP with residual connections; non-arithmetic operations (softmax, GELU/SiLU, RMSNorm) are the parts that resist direct circuit compilation.

\subsection{Zero-Knowledge Proofs}
\label{sec:zkp-background}

A ZK proof system $(\prover, \verifier)$ for relation $\mathcal{R}$ satisfies completeness, soundness, and zero-knowledge~\cite{goldwasser1989knowledge}.
We instantiate \method{} on Halo2 with the Inner Product Argument (IPA) commitment scheme~\cite{halo2,bunz2018bulletproofs}: no trusted setup, $\bigO(\log n)$ proof size, and discrete-log security---pairing-free and transparent, unlike pairing-based alternatives~\cite{groth16}, though DLog-based IPA is not post-quantum secure.

\subsection{Threat Model}
\label{sec:threat}

\begin{definition}[Adversary and Security Goals]
\label{def:threat}
The adversary $\adv$ is a PPT-bounded, adaptive malicious provider that controls the inference infrastructure and may substitute the model, apply unannounced quantization/pruning, fabricate outputs, or forge proofs; in the \emph{full-verification regime} $\adv$ may observe the entire verifier protocol before choosing which layers to corrupt, and is constrained only by standard cryptographic assumptions (discrete-log hardness, collision-resistant hashing).
Model identity is bound by a public Merkle root $c_W$ over canonical $W_{1:L}$ serialisation, supplied as a public input to every layer circuit and checked against the announced identity by the verifier.
\emph{Integrity:} for any PPT $\adv$ producing $(\pi^*, y^*, c_W)$ with $y^* \neq f_W(x)$, $\Pr[\verifier(\pi^*, x, y^*, c_W)=\text{accept}] \leq \negl(\lambda)$.
\emph{Verifier-facing privacy:} weights $W_\ell$ and activations $h_\ell$ stay hidden from verifiers and auditors; $x$ is committed via $c_0 = \mathcal{H}(\text{Embed}(x))$ and may be a private witness (auditor sees only $c_0$) or a public input.
The triage of \S\ref{sec:fisher} is a strict relaxation: it adds the assumption that the layer-selection strategy is hidden from $\adv$ (or that probabilistic detection is acceptable) and is presented as an efficiency tool, not a security claim.
\end{definition}

\textbf{Out of scope.}
\method{} does \emph{not} hide $x$ from the prover: the provider executes inference on $x$ in cleartext exactly as in a standard cloud LLM API.
Hiding $x$ from the provider is the complementary goal of HE~\cite{gilad2016cryptonets,juvekar2018gazelle} and MPC~\cite{mohassel2017secureml,riazi2018chameleon}; \method{} composes with such techniques but does not replace them.
We also do not address DoS, side-channels, training verification, or attacks on the model itself (adversarial examples, data poisoning).

\section{Layerwise Proof Framework}
\label{sec:layerwise}

The central challenge in ZK-proving LLM inference is scale: a GPT-2-Small forward pass involves $\sim$124M parameters, and a monolithic circuit yields $>10^9$ constraints with $\bigO(N\log N)$ prover time plus $\bigO(n\cdot d)$ cross-layer routing constraints per boundary that serve no computational purpose.
The central observation is that transformer inference has inherent \emph{layer independence}: each block depends only on the previous layer's output and its own parameters, enabling decomposition into independently provable units linked by cryptographic commitments.
Each layer circuit has at most $\max_\ell n_\ell$ constraints, reducing peak memory by a factor of $L$ and eliminating all cross-layer routing.
Figure~\ref{fig:architecture} shows the resulting pipeline.

\begin{figure}[!htbp]
\centering
\includegraphics[width=\linewidth]{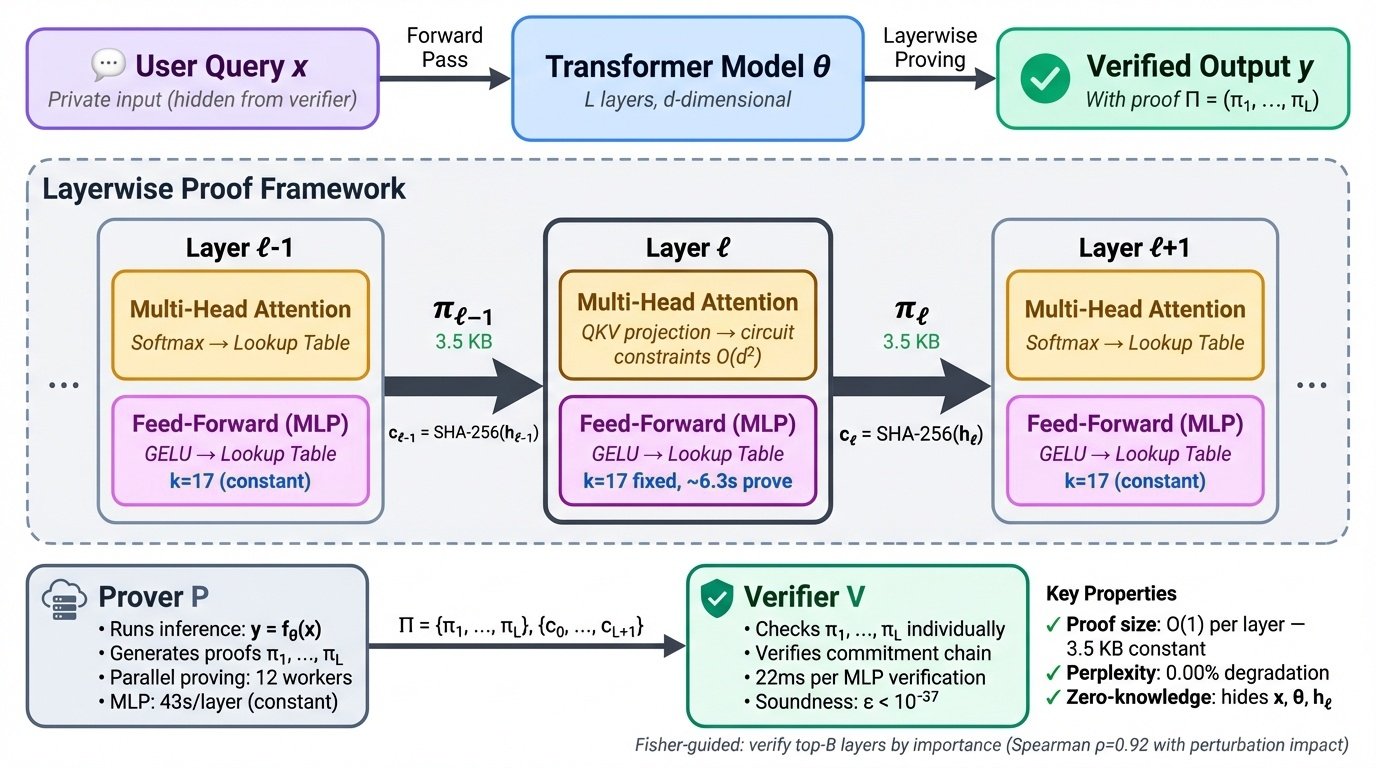}
\caption{System architecture. Input $x$ flows through $L$ transformer layers, each decomposed into independently provable Attention and MLP sub-circuits with 16-bit lookup-table approximations. Each sub-circuit emits a constant-size proof ($\sim$3.5\,KB) linked by a SHA-256 commitment chain $c_0 \to c_1 \to \cdots \to c_{L+1}$; the Verifier checks individual proofs and chain integrity.}
\label{fig:architecture}
\end{figure}

\subsection{Proof Decomposition}

\textbf{Layer proof.} For transformer layer $\ell$ with input $h_{\ell-1} \in \R^{n \times d}$, output $h_\ell$, and parameters $W_\ell$, a layer proof $\pi_\ell$ demonstrates $h_\ell = f_\ell(h_{\ell-1}; W_\ell)$ where $f_\ell$ denotes the layer computation (attention + FFN + residuals), and includes commitments $c_{\ell-1} = \mathcal{H}(h_{\ell-1})$, $c_\ell = \mathcal{H}(h_\ell)$ where $\mathcal{H}$ is collision-resistant (SHA-256 in our implementation). For presentation we write $c_\ell = \mathcal{H}(h_\ell)$; Appendix~\ref{app:soundness} formalises this as $c_\ell = \mathcal{H}(c_{\ell-1} \,\|\, \mathrm{enc}(B_\ell) \,\|\, \ell)$, where $B_\ell$ is an in-circuit-opened algebraic commitment to the private activation $h_\ell$.

\textbf{Where the chain hash lives.}
\method{} computes $c_\ell$ \emph{externally} on the host and supplies it to layer circuit~$\ell$ as a \emph{public input}, avoiding the ${\sim}25$K-constraint penalty of in-circuit SHA-256~\cite{halo2}. Conceptually we write the chain marker as $c_\ell = \mathcal{H}(h_\ell)$; in the formal protocol (Appendix~\ref{app:soundness}) the host hashes the public encoding of an algebraic boundary commitment $B_\ell$ to the activation, i.e.\ $c_\ell = \mathcal{H}(c_{\ell-1} \,\|\, \mathrm{enc}(B_\ell) \,\|\, \ell)$, while the layer circuit proves, via the Halo2 IPA polynomial commitment, that its private witness $h_\ell$ opens to $B_\ell$. The same external-commit pattern binds weights: $c_W$ is a Merkle root over canonical $W_{1:L}$ serialisation, also a public input, with a Merkle opening of $W_\ell$ enforced in each sub-circuit. The composite proof $\Pi = (\pi_0, \pi_1, \ldots, \pi_{L+1})$ forms a chain $c_\ell^{\text{out}}(\pi_\ell) = c_\ell^{\text{in}}(\pi_{\ell+1})$ for all $\ell \in [L]$.

\subsection{Soundness and Zero-Knowledge}
\label{sec:soundness}

\begin{lemma}[Chain Binding]
\label{lem:chain-binding}
For two accepting composite proofs $\Pi, \Pi'$ with $c_0 = c'_0$ that yield different intermediate states at some layer, $\Pr[\text{both verify}] \leq (L{+}2)\cdot\negl(\lambda)$.
\end{lemma}
\emph{Proof.} At any layer $\ell$, accepting proofs with $c_\ell = c'_\ell$ but $h_\ell \neq h'_\ell$ require either a SHA-256 collision or an IPA-binding break (each $\leq \negl(\lambda)$ under discrete-log hardness); union-bound over the $L{+}2$ positions.

\begin{theorem}[Compositional Soundness]
\label{thm:soundness}
With per-layer soundness $\epsilon_\ell$ and collision-resistant $\mathcal{H}$, the composite system has soundness error $\epsilon_{\text{total}} \leq \sum_{\ell=0}^{L+1} \epsilon_\ell + 2(L{+}2)\cdot\negl(\lambda)$.
\end{theorem}

\emph{Proof sketch.}
At the first divergence layer $\ell^*$, either $\pi_{\ell^*}^*$ proves a false relation or two distinct witnesses share $c_{\ell^*}$---a Lemma~\ref{lem:chain-binding} event. Assuming $\epsilon_\ell\leq 2^{-\lambda}$, $\epsilon_{\text{total}}\leq 3(L{+}2)\cdot 2^{-\lambda}$; at $\lambda{=}128$, $L{=}32$, $\epsilon_{\text{total}}\leq 102\cdot 2^{-128}\approx 3\times 10^{-37}$.
\textbf{Zero-knowledge.} Each layer admits a simulator $\simulator_\ell(c_{\ell-1}, c_\ell)$; sequential composition yields $\simulator$ indistinguishable from $\Pi$. Hiding of the chain $\{c_\ell\}$ holds under the random-oracle modelling of $\mathcal{H}$ provided each $h_\ell$ retains $\omega(\log\lambda)$ bits of min-entropy conditioned on $\adv$'s view---inherited whenever the prompt distribution does.

\subsection{Complexity Analysis}

Each layer proof has size $|\pi_\ell| = \bigO(\log n_\ell)$, independent of $d$, and is 3.2--3.7\,KB ($\sim$6.9\,KB per layer up to $d{=}768$). Layer independence enables parallel proving: with $P=L{+}2$ workers, wall-clock is $T_{\text{forward}}+\max_\ell T_{\text{prove}}(\ell)$. The SHA-256 chain is 32\,B per layer at $\sim$192\,$\mu$s per hash---negligible against proving time.

\subsection{Proof Generation and Verification}

Algorithm~\ref{alg:layerwise} summarises the protocol. The forward pass is sequential; proving and verification are embarrassingly parallel across all $L{+}2$ sub-circuits; no SHA-256 invocation appears inside any sub-circuit.

\begin{algorithm}[!htbp]
\caption{Layerwise Proof Generation and Verification (schematic; Appendix~\ref{app:soundness} gives the boundary-commitment version used for the formal proof).}
\label{alg:layerwise}
\begin{algorithmic}[1]
\STATE \textbf{Prover.} \emph{Witness:} $W_{1{:}L}, x$. \emph{Public:} $c_W, \Pi, y, \{c_0,\ldots,c_{L+1}\}$.
\STATE \quad Forward+commit: $h_0{\gets}\text{Embed}(x)$, $c_0{\gets}\mathcal{H}(h_0)$; for $\ell{=}1{:}L$: $h_\ell{\gets}\text{Block}_\ell(h_{\ell-1})$, $c_\ell{\gets}\mathcal{H}(h_\ell)$; $y{\gets}\text{LMHead}(\text{Norm}(h_L))$, $c_{L+1}{\gets}\mathcal{H}(y)$.
\STATE \quad In parallel for $\ell\in\{0,\ldots,L{+}1\}$: $\pi_\ell\gets\text{Prove}_\ell(\text{wit}: h_{\ell-1},h_\ell,W_\ell;\, \text{pub}: c_W, c_{\ell-1},c_\ell)$, each circuit additionally enforcing the opening of $W_\ell$ against $c_W$.
\STATE \textbf{Verifier.} \emph{Input:} $c_W, \Pi, y, \{c_\ell\}$.
\STATE \quad Check $c_W$ against the announced model identity. For each $\ell$: check $\text{Verify}_\ell(\pi_\ell;c_W,c_{\ell-1},c_\ell)$ and $c_\ell^{\text{out}}(\pi_\ell)=c_\ell^{\text{in}}(\pi_{\ell+1})$; check $c_{L+1}=\mathcal{H}(y)$. \textsc{accept} iff all pass.
\end{algorithmic}
\end{algorithm}

\section{ZK-Friendly Approximations}
\label{sec:approx}

Softmax exponentials, GELU/SiLU error functions, and RMSNorm/LayerNorm divisions have no direct arithmetic-circuit form. Polynomial approximations accumulate error; we instead use 16-bit lookup tables whose end-to-end accuracy impact is below our measurement precision ($\Delta\text{PPL}<10^{-4}$, \S\ref{sec:verifier-cost}). For each function $f$, $\text{LUT}_f = \{(i, f(x_i)) : i\in[0,2^{16})\}$ over its operating range (e.g.\ $[-8,8]$ for softmax exp) is enforced in-circuit by Plookup~\cite{halo2}: one constraint per lookup. We cover softmax (in-circuit division via $a=b\cdot c$), GELU/SiLU, and RMSNorm rsqrt (sum-of-squares in-circuit, rsqrt looked up); LayerNorm reduces to RMSNorm plus mean subtraction.

\subsection{Approximation Error and Comparison}
\label{sec:approx-error}

Table~\ref{tab:approx-error} reports per-operation errors measured over 100{,}000 sampled points (softmax measured end-to-end over 496{,}000 output probabilities, sequence lengths 8--128).

\begin{table}[!htbp]
\centering
\caption{Lookup table approximation errors (16-bit precision, real measurements).}
\label{tab:approx-error}
\small
\begin{tabular}{@{}lccccc@{}}
\toprule
\textbf{Function} & \textbf{Mean Rel.} & \textbf{P99 Rel.} & \textbf{Mean Abs.} & \textbf{Table Size} & \textbf{Constraints} \\
\midrule
GELU & 0.026\% & 0.158\% & $4.1 {\times} 10^{-5}$ & 65,536 & 1 per lookup \\
SiLU & 0.021\% & 0.158\% & $6.5 {\times} 10^{-5}$ & 65,536 & 1 per lookup \\
Softmax (e2e)$^\dagger$ & --- & --- & $7.2 {\times} 10^{-6}$ & 65,536 & 1 per lookup \\
RMSNorm (rsqrt) & 0.003\% & 0.034\% & $7.5 {\times} 10^{-5}$ & 65,536 & 1 per lookup \\
\bottomrule
\end{tabular}

\smallskip
\raggedright\scriptsize $^\dagger$Softmax relative error is ill-defined for near-zero probabilities; we report absolute error instead. Mean absolute error of $7.2{\times}10^{-6}$ corresponds to negligible output distortion.
\end{table}

\textbf{Error Propagation Through Layers.}
A key concern is whether per-operation errors compound when all layers use lookup approximations.
On a 12-layer transformer ($d{=}256$), KL divergence between exact and approximated outputs remains below $10^{-8}$ even with all 12 layers approximated, and L2 RMS error grows sub-linearly from $5{\times}10^{-5}$ (1 layer) to $1.4{\times}10^{-4}$ (12 layers).
This sub-linear accumulation confirms that residual connections absorb per-layer errors, explaining the zero perplexity degradation observed end-to-end.

\subsection{Comparison with Lookup-Based Approximations in Prior ZKML}
\label{sec:lut-comparison}

Several prior ZKML systems also handle non-arithmetic operations via lookups or tabulated polynomial approximations.
We compare design choices and the resulting accuracy in Table~\ref{tab:lut-compare}.

\begin{table}[!htbp]
\centering
\caption{Lookup/approximation strategies in ZKML systems (numbers cited from each paper unless noted). LN = LayerNorm; sm = softmax.}
\label{tab:lut-compare}
\small
\setlength{\tabcolsep}{3pt}
\begin{tabular}{@{}l l c l c@{}}
\toprule
\textbf{System} & \textbf{Strategy} & \textbf{Prec.} & \textbf{Ops covered} & \textbf{$\Delta$PPL} \\
\midrule
zkCNN~\cite{zhang2018zkcnn} & no LUT & --- & none (CNN) & --- \\
ZKML~\cite{kang2024zkml} & Halo2 LUT & mixed & ReLU, GELU & ${<}0.5\%$ \\
zkLLM~\cite{sun2024zkllm} & tlookup/zkAttn & 8-bit & sm, GELU & ${<}0.5\%$ \\
zkGPT~\cite{qu2025zkgpt} & GKR poly. & --- & sm, GELU, LN & ${\sim}0.05$--$0.5\%$ \\
\midrule
\textbf{\method{}} & Plookup LUT & \textbf{16-bit} & sm, GELU, SiLU, RMSNorm & $\boldsymbol{<10^{-4}}$ \\
\bottomrule
\end{tabular}
\end{table}

Compared to prior lookup-based ZKML, \method{} uses 16-bit precision tables (256$\times$ more entries than zkLLM's 8-bit tlookup) at the same per-lookup circuit cost (one Plookup constraint), absorbing the residual $10^{-5}$-scale errors via the transformer's residual structure (\S\ref{sec:approx-error}) for end-to-end $\Delta\text{PPL}<10^{-4}$ across three models and two datasets (\S\ref{sec:experiments}).

\section{Fisher-Guided Audit-Budget Triage}
\label{sec:fisher}

\begin{wraptable}{r}{0.50\linewidth}
\vspace{-1.2em}
\centering
\caption{Fisher-guided vs.\ random layer selection at 50\% budget (WikiText-2).}
\label{tab:fisher}
\small
\setlength{\tabcolsep}{2pt}
\begin{tabular}{@{}lrrrr@{}}
\toprule
\textbf{Model} & $L$ & \textbf{Fisher} & \textbf{Random} & $\Delta$ (pp) \\
\midrule
GPT-2 & 12 & 64.7\% & 54.3\% & +10.4 \\
TinyLLaMA & 22 & 86.0\% & 79.3\% & +6.7 \\
Phi-2 & 32 & 62.4\% & 50.6\% & +11.8 \\
\midrule
\multicolumn{4}{r}{\textbf{Average:}} & \textbf{+9.6} \\
\bottomrule
\end{tabular}
\vspace{-1.0em}
\end{wraptable}
\begin{sloppypar}
Under the relaxed regime of Definition~\ref{def:threat} (auditor keeps the layer-selection strategy hidden from $\adv$, or accepts probabilistic detection), an audit with budget $B<L{+}2$ benefits from verifying the most sensitive layers first. We use the scalar score $I_\ell = \tr(\mathcal{F}_\ell)/|\theta_\ell|$, where $\mathcal{F}_\ell$ is the per-layer Fisher information~\cite{fisher1925theory,kirkpatrick2017overcoming}, and verify the top-$B$. On GPT-2, Fisher scores correlate with Gaussian-perturbation impact ($\rho{=}0.888$ at $\sigma{=}0.01$, $\rho{=}0.916$ at $\sigma{=}0.05$, $p<10^{-4}$). At 50\% budget, Fisher covers +9.6\,pp more total sensitivity than random (Table~\ref{tab:fisher}); rankings are stable across WikiText-2/C4/Pile-v2 ($\rho{=}0.55$--$0.95$).
\end{sloppypar}


\section{Experiments}
\label{sec:experiments}

\textbf{Setup.}
We evaluate on GPT-2 variants ($d\in\{64,128,256,512,768\}$, up to GPT-2-Small 124M), plus GPT-2-Medium and TinyLLaMA-1.1B~\cite{touvron2023llama} for accuracy, on a single Intel Xeon CPU at 2.4\,GHz with 64\,GB RAM (no GPU). The Halo2 IPA backend is in Rust over the Pallas scalar field with 16-bit fixed-point LUTs; perplexity uses WikiText-2~\cite{merity2016pointer} and C4 (seq.\ 128). EZKL baselines on the same CPU: 19.3\,s/26.0\,KB for MLP ($d{=}64$) and 19.5\,s/39.5\,KB for attention ($d{=}16$); EZKL beyond these dimensions exceeds 64\,GB during circuit compilation. The polynomial baseline of \S\ref{sec:verifier-cost} is a degree-7 minimax Chebyshev fit over the same operating range as our LUTs.

\subsection{End-to-End Proof Performance}

Figure~\ref{fig:proving-scaling}(a) reports proving-time scaling for attention and MLP sub-circuits.

\begin{figure}[!htbp]
\centering
\includegraphics[width=\linewidth]{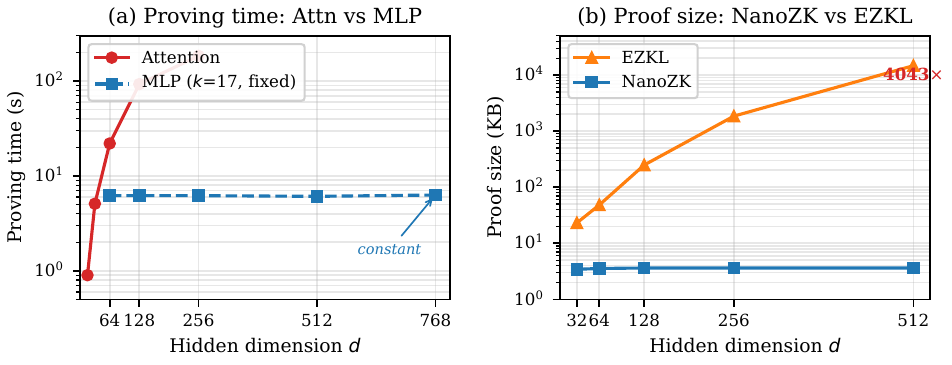}
\caption{(a)~Proving time: attention sub-circuit prove-only grows with $d$ while MLP stays constant at $\sim$6.3\,s (fixed degree $k{=}17$). (b)~Per-sub-circuit proof size: \method{} stays constant at $\sim$3.6\,KB; EZKL's monolithic proof grows to $>$14\,MB at $d{=}512$ (EZKL stops there due to circuit-compilation OOM on 64\,GB).}
\label{fig:proving-scaling}
\end{figure}

We highlight three results.
First, \textbf{proof size is constant}: each sub-circuit (attention or MLP) produces a 3.2--3.7\,KB proof regardless of dimension.
This property, inherited from the IPA commitment scheme, means proof size does not leak information about model architecture beyond what is explicitly disclosed.
Second, \textbf{MLP proving time is constant} at $\sim$6.3\,s (proving only) across $d{=}64$ to $d{=}768$, because the circuit degree $k{=}17$ is fixed---constraint rows are pre-allocated, and smaller models simply leave rows unused.
Third, \textbf{attention proving time scales} with $d$ because the attention circuit's degree $k$ must grow to accommodate $\bigO(d^2)$ constraints from the QKV projections and score matrix.
Table~\ref{tab:attn-scaling} reports real attention proof times alongside end-to-end full-block (Attn+MLP) proofs measured up to $d{=}128$ with a GPU projection at GPT-2 scale.

\begin{table}[!htbp]
\centering
\caption{Attention sub-circuit scaling and full transformer block end-to-end proof time (real Halo2 IPA, prove-only; setup is amortisable across queries with the same circuit shape). $^\dagger$Extrapolated to GPT-2 scale; $^\ddagger$GPU-projected via Table~\ref{tab:comparison-gpu}. Block E2E at $d{=}256$ is omitted: the combined attention+MLP circuit exceeds 64\,GB during witness generation on our CPU.}
\label{tab:attn-scaling}
\small
\begin{tabular}{@{}rrrrrrr@{}}
\toprule
$d$ & $k$ & \textbf{Setup (s)} & \textbf{Attn Prove (s)} & \textbf{Verify (ms)} & \textbf{Block E2E (s)} & \textbf{Proof (KB)} \\
\midrule
16 & 14 & 3.3 & 0.9 & 13 & 6.1 & 5.2 \\
32 & 17 & 29.0 & 5.1 & 95 & 12.3 & 5.3 \\
64 & 19 & 122.8 & 22.0 & 365 & 39.5 & 5.5 \\
128 & 21 & 526.8 & 92.9 & 1{,}747 & 94.2 & 5.7 \\
256 & 22 & 1{,}018 & 184.3 & 3{,}396 & --- & 5.8 \\
768$^\dagger$ & 24 & --- & $\sim$25$^\ddagger$ & --- & $\sim$68$^\ddagger$ & $\sim$6.9 \\
\bottomrule
\end{tabular}
\end{table}

Attention proving is the scalability bottleneck. At $d{=}768$, MLP prove-only is 6.3\,s ($\sim$26$\times$ NumPy) and setup+prove is 43\,s for a single query; the two figures differ only by whether setup is amortised across queries with the same circuit shape. 12-layer GPT-2-Small projects to $\sim$14\,min sequentially. The measured CPU parallel speedup at $L{+}2$ workers is 2.66$\times$, capped by shared-L3 contention between MSMs; a discrete-GPU implementation with one MSM per device avoids this and we project $\sim$2\,min at 12 GPU workers, an L3-shared GPU would inherit the 2.66$\times$ cap.

\subsection{Comparison and GPU Projection}

\begin{wraptable}{r}{0.50\linewidth}
\vspace{-1.2em}
\centering
\caption{Comparison with recent ZKML systems at GPT-2 scale, prior work as reported. ZKML, zkGPT on 32-thread CPU; zkLLM on A100; \method{} on single CPU. $^\dagger$MLP-only setup+prove; full block at $d{=}768$ projects to ${\sim}68$\,s/block.}
\label{tab:comparison-gpu}
\small
\setlength{\tabcolsep}{3pt}
\begin{tabular}{@{}lccl@{}}
\toprule
\textbf{System} & \textbf{Prove} & \textbf{Size} & \textbf{$\Delta$PPL} \\
\midrule
ZKML~\cite{kang2024zkml} & 67\,min & 7.8\,KB & ${<}0.5\%$ \\
zkLLM~\cite{sun2024zkllm} & 15.8\,s & 126\,KB & ${<}0.5\%$ \\
zkGPT~\cite{qu2025zkgpt} & 21.8\,s & 101\,KB & ${<}0.5\%$ \\
\method{} & 43\,s$^\dagger$\,/\,68\,s & \textbf{6.9\,KB} & $\boldsymbol{<10^{-4}}$ \\
\bottomrule
\end{tabular}
\vspace{-1.0em}
\end{wraptable}
\begin{sloppypar}
Table~\ref{tab:comparison-gpu} compares \method{} with prior ZKML systems using their reported numbers; no direct same-hardware comparison is possible since zkGPT and zkLLM have no public implementation at writing.
\method{}'s per-sub-circuit proof (3.5\,KB) is the smallest reported; the composite at $L{=}12$ is $\sim$83\,KB, $\sim$$1.2$--$1.5\times$ smaller than zkLLM/zkGPT's monolithic 101--126\,KB, while enabling parallel proving, independent storage, and selective re-verification. EZKL's monolithic proofs grow to $>$14\,MB at $d{=}512$ on the same CPU. Icicle~\cite{icicle2023} reports $30$--$50\times$ MSM speedup on A100 at $n\geq2^{20}$ with diminishing returns at smaller $n$; applied to our $\bigO(d^2)$ CPU baseline at a conservative $15$--$30\times$ effective speedup over $k\in\{20,22,24\}$, this projects (not measured) to $\sim$25\,s for attention at $d{=}768$ and $\sim$68\,s per layer (Table~\ref{tab:attn-scaling}), $\sim$2\,min for 12-layer GPT-2 with 12 GPU workers.
\end{sloppypar}

\subsection{Verifier Cost and Accuracy}
\label{sec:verifier-cost}

\textbf{Verifier wall-clock.}
MLP verification is constant at $\sim$22\,ms; attention verification scales 13\,ms ($d{=}16$) to 3.4\,s ($d{=}256$) (Table~\ref{tab:attn-scaling}); digest checks are $<$1\,$\mu$s each. A measured 12-layer verifier ($d\leq 256$) is 41\,s sequential / 4\,s with 12 threads; at $d{=}768$ the projection is $\sim$6\,min sequential or $\sim$30\,s with 12 threads. Verification asymptotically matches a monolithic IPA verifier and recovers parallelism it cannot offer.

\begin{sloppypar}
\textbf{Accuracy.}
We report relative perplexity change $\Delta\text{PPL}$ under lookup vs.\ polynomial approximations on three models (GPT-2, GPT-2-Medium, TinyLLaMA) and two datasets (WikiText-2, C4; seq.\ 128). Lookup-table $\Delta\text{PPL}$ is below our measurement precision ($\!<10^{-4}$) in all six combinations; the degree-7 polynomial baseline (\S\ref{sec:experiments}~Setup) degrades PPL by 106--136\%---much wider than the 0.05--0.5\% of co-designed polynomial systems like zkGPT~\cite{qu2025zkgpt}, where activation range and polynomial degree are jointly tuned; our naive same-range fit deliberately does not match this. The $10^{-5}$-scale quantisation errors (Table~\ref{tab:approx-error}) are absorbed by the residual structure described above.
\end{sloppypar}

\section{Related Work}
\label{sec:related}

\textbf{Zero-Knowledge Machine Learning.}
EZKL~\cite{ezkl2023} compiles ONNX to PLONK but struggles beyond 1M parameters.
ZKML~\cite{kang2024zkml} achieves GPT-2 proofs via Halo2 compilation ($\sim$1\,h CPU).
zkLLM~\cite{sun2024zkllm} (tlookup/zkAttn, 15\,min on A100) and zkGPT~\cite{qu2025zkgpt} (GKR, sub-25\,s) also use per-layer commitments---the pattern is not new. \method{}'s specific contribution is pairing \emph{external} SHA-256 chain commitments with Halo2 IPA public-input binding for both $h_\ell$ and $W_\ell$, avoiding ${\sim}25$K in-circuit hashing constraints per layer in exchange for IPA binding, and yielding constant 3.2--3.7\,KB per-sub-circuit proofs that are independently verifiable and selectively re-checkable. Noir~\cite{noir2022}, Plonky2~\cite{plonky22022}, ZEXE~\cite{bowe2018zexe} are general toolchains; zkCNN~\cite{zhang2018zkcnn} targets CNNs.

\textbf{Privacy-Preserving ML and Verifiable Computation.}
HE~\cite{gilad2016cryptonets,juvekar2018gazelle} protects inputs but offers no integrity; MPC~\cite{mohassel2017secureml,riazi2018chameleon} needs non-collusion; DP~\cite{dwork2014algorithmic} protects training statistics. SNARKs~\cite{groth16} and STARKs~\cite{ben2019aurora} face scalability at LLM scale; Bulletproofs~\cite{bunz2018bulletproofs} have linear verification. \method{} delivers privacy and integrity under a single cryptographic-hardness assumption and sidesteps the constraint-system scaling barrier.

\textbf{Recursive proofs and IVC.}
Folding-based IVC (Nova~\cite{nova2022}) is the natural complement: a Nova fold over $\Pi$ would compress the $\sim$83\,KB composite into a single proof, at the cost of an in-circuit recursive verifier and a uniform circuit shape. \method{}'s heterogeneous, parallelisable sub-proofs occupy the opposite point of the design space and the two compose cleanly; with low-rank/kernelised attention~\cite{wang2020linformer,choromanski2021performer}, the layerwise-then-fold pattern is a plausible path to verifiable inference at 7\,B and beyond.

\textbf{Conclusion.}
\method{}'s layerwise composition gives constant 3.2--3.7\,KB sub-circuit proofs, $\epsilon\approx3\times10^{-37}$ soundness, and ZK at minute-scale CPU wall-clock (6.3\,s MLP measured; $\sim$25\,s attention at $d{=}768$ \emph{projected} via GPU MSM); ZK-friendlier attention + Nova-style folding extends the design to LLM scale.
\emph{Acknowledgments.} Thanks to ICICS reviewers for feedback that improved the privacy claim, commitment-chain semantics, and IVC comparison.

\bibliographystyle{splncs04}
\bibliography{references}

\newpage
\appendix

\begin{center}
\Large\bfseries Appendices (arXiv-only extended material)
\end{center}
\medskip

\noindent The remainder of this document is an extended technical companion to the 12-page ICICS~2026 camera-ready that precedes it. The appendices are included only in this arXiv preprint, not in the Springer LNCS proceedings. Cross-references to body sections (e.g.~Section~\ref{sec:layerwise}) point back into the camera-ready; appendix-internal references are prefixed \texttt{app:}.

\medskip
\noindent\textbf{Roadmap.}
\begin{itemize}[topsep=2pt,itemsep=1pt]
\item Appendix~\ref{app:soundness}: full soundness and zero-knowledge proofs.
\item Appendix~\ref{app:lookup}: lookup-table approximation construction and error propagation.
\item Appendix~\ref{app:circuits}: Halo2 circuit constructions and chip designs.
\item Appendix~\ref{app:algorithms}: prover/verifier pseudocode, Fiat-Shamir spec, audit-budget triage.
\item Appendix~\ref{app:experiments}: extended experimental tables and ablations.
\item Appendix~\ref{app:threat}: per-mode threat-model breakdown.
\item Appendix~\ref{app:gpu}: GPU projection methodology and measured/projection separation.
\item Appendix~\ref{app:repro}: reproducibility checklist, hardware/software versions, build instructions.
\end{itemize}

\section{Full Soundness and Zero-Knowledge Proofs}
\label{app:soundness}

This appendix expands Theorem~\ref{thm:soundness} and the ZK-simulator sketch of Section~\ref{sec:soundness}. Concretely, we (i)~make precise the algebraic mechanism by which the SHA-256 commitment chain binds \emph{private} activations $h_\ell$ via an in-circuit-checked algebraic boundary commitment $B_\ell$, (ii)~give the full hybrid argument with explicit loss terms, and (iii)~exhibit a polynomial-time ZK simulator under the random-oracle model.

\subsection{Preliminaries and Notation}
\label{app:soundness:prelim}

We fix a security parameter $\lambda$ and work over the Pallas scalar field $\F_q$ with $q \approx 2^{255}$; the prime-order curve is $\mathbb{G}_{\mathrm{Pallas}}$. The proof system $(\prover,\verifier)$ is Halo2 with the Inner Product Argument (IPA)~\cite{halo2,bunz2018bulletproofs}; the discrete-log assumption is taken in $\mathbb{G}_{\mathrm{Pallas}}$. Throughout this appendix, $L$ denotes the number of transformer blocks, $H = \mathrm{SHA\text{-}256}$, $\mathrm{IV}$ is a fixed initial chain value, and $c_W$ denotes the publicly announced model commitment (Merkle root over canonical $W_{1{:}L}$ serialisation; see Section~\ref{sec:layerwise}). We index layers $\ell \in \{0, 1, \ldots, L{+}1\}$, where layer~$0$ is the embedding step ($h_0 = \mathrm{Embed}(x; W_{\mathrm{emb}})$) and layer~$L{+}1$ is the LM head ($y = \mathrm{LMHead}(\mathrm{Norm}(h_L); W_{\mathrm{head}})$). This matches body Algorithm~\ref{alg:layerwise}, which constructs $\Pi = (\pi_0, \pi_1, \ldots, \pi_{L+1})$ with $c_0 = \mathcal{H}(\mathrm{Embed}(x))$ and $c_{L+1} = \mathcal{H}(y)$.

\paragraph{Boundary commitments: clarifying the binding mechanism.}
A technical clarification not made explicit in the body: \emph{the chain binds activations via an algebraic, in-circuit-checkable boundary commitment $B_\ell$, not directly via the SHA-256 digest $H(h_\ell)$.} Concretely, fix Halo2-IPA's generator vector $\mathbf G$ and a blinding generator $U$; for $h_\ell \in \F_q^{N}$ ($N \coloneqq nd$) and uniform $r_\ell \xleftarrow{\$} \F_q$, define
\[
B_\ell \;\coloneqq\; \mathrm{Com}(h_\ell;\, r_\ell) \;=\; \langle h_\ell, \mathbf G\rangle + r_\ell\,U \;\in\; \mathbb{G}_{\mathrm{Pallas}}.
\]
$B_\ell$ is a single Pallas curve point ($\approx 64$\,B serialised), is computationally binding under DLog and perfectly hiding given uniform $r_\ell$, and is exactly the object Halo2-IPA commits to internally when the activation is laid out as an advice polynomial. We expose this commitment as a \emph{public-input curve point}; the layer circuit enforces (via an IPA opening gate) that the prover's advice column for $h_\ell$ opens to the publicised $B_\ell$. The body's notation $c_\ell = \mathcal{H}(h_\ell)$ is the chain marker associated with layer~$\ell$; its formal binding to private $h_\ell$ proceeds through $B_\ell$, as made precise below.

\paragraph{Chain digests on public bytes.}
With $B_\ell$ in hand, the SHA-256 chain is defined over \emph{public bytes} $\mathrm{enc}(B_\ell)$:
\[
c_{-1} \coloneqq \mathrm{IV}, \qquad c_\ell \;\coloneqq\; H\!\big(c_{\ell-1} \,\|\, \mathrm{enc}(B_\ell) \,\|\, \ell\big), \qquad \ell = 0, 1, \ldots, L{+}1.
\]
Each layer proof $\pi_\ell$ takes $(c_W, B_{\ell-1}, B_\ell, \ell)$ as Halo2 public inputs, and the layer relation $\mathcal{R}_\ell$ enforces (in-circuit) openings of advice columns to $B_{\ell-1}$ and $B_\ell$ together with the per-layer transformer relation (embedding for $\ell{=}0$, transformer block for $\ell \in \{1,\ldots,L\}$, LM head for $\ell{=}L{+}1$) and a Merkle opening of the relevant weight slice against $c_W$. SHA-256 is computed by the verifier externally on the public-input bytes and is never instantiated inside the constraint system.

\subsection{Theorem A.1 (Compositional Soundness)}
\label{app:soundness:thm}

\begin{theorem}[Full Statement]
\label{thm:full-soundness}
Let $\Pi = \big(c_W,\; (B_\ell)_{\ell=0}^{L+1},\; (c_\ell)_{\ell=0}^{L+1},\; (\pi_\ell)_{\ell=0}^{L+1},\; y\big)$ be a layerwise proof bundle and let $x$ be the public input. Assume:
\begin{enumerate}[label=(\roman*),topsep=2pt,itemsep=1pt]
\item Halo2-IPA satisfies $\epsilon_{\mathrm{IPA}}$-knowledge soundness against PPT adversaries under DLog over $\mathbb{G}_{\mathrm{Pallas}}$;
\item the Pedersen-style boundary commitment $\mathrm{Com}$ is $\epsilon_{\mathrm{bind}}$-binding under DLog;
\item SHA-256 is $(t,\epsilon_H)$-collision-resistant for $t = \poly(\lambda)$;
\item each layer relation $\mathcal{R}_\ell$ encodes faithful transformer arithmetic for that layer (embedding for $\ell{=}0$, block for $\ell \in \{1,\ldots,L\}$, LM head for $\ell{=}L{+}1$), in-circuit openings to $B_{\ell-1}$ and $B_\ell$, and a Merkle opening of the relevant weights against $c_W$.
\end{enumerate}
The verifier accepts iff (a)~$\pi_\ell$ verifies against $(c_W, B_{\ell-1}, B_\ell, \ell)$ for every $\ell$, (b)~$c_\ell = H(c_{\ell-1} \,\|\, \mathrm{enc}(B_\ell) \,\|\, \ell)$ for every $\ell$ with $c_{-1} = \mathrm{IV}$, and (c)~$y$ is consistent with the opening committed by $B_{L+1}$ (since $y$ is public, this is checked by re-committing $y$ with the prover-supplied $r_{L+1}$ and comparing).

Then for every PPT $\adv$,
\[
\Pr\!\left[
\begin{array}{c}
\verifier\text{ accepts }\Pi \;\wedge\; \big(\,h_0 \ne \mathrm{Embed}(x; W_{\mathrm{emb}})\\[2pt]
\quad\text{or}\;\, \exists\,\ell{\in}\{1,\dots,L\}:\; h_\ell \ne \mathrm{Block}_\ell^{W_\ell}(h_{\ell-1})\\[2pt]
\quad\text{or}\;\, y \ne \mathrm{LMHead}(\mathrm{Norm}(h_L); W_{\mathrm{head}})\,\big)
\end{array}
\right]
\;\le\; (L{+}2)\,(\epsilon_{\mathrm{IPA}} + \epsilon_{\mathrm{bind}} + \epsilon_H).
\]
\end{theorem}

\begin{proof}[Proof by hybrid argument]
We track loss across $L{+}2$ layers ($\ell = 0,\ldots,L{+}1$).

\textbf{Hybrid $\mathsf{H}_0$.} The real experiment: $\adv$ outputs $\Pi$ and $\verifier$ accepts.

\textbf{Hybrid $\mathsf{H}_1$ (IPA extraction).} For each $\ell$, run the Halo2-IPA knowledge extractor on $\pi_\ell$. With probability $\ge 1 - \epsilon_{\mathrm{IPA}}$ it returns a witness $w_\ell^* = (h_{\ell-1}^*, h_\ell^*, W_\ell^*, r_{\ell-1}^*, r_\ell^*, \text{Merkle-path}^*)$ satisfying $\mathcal{R}_\ell(c_W, B_{\ell-1}, B_\ell, \ell;\, w_\ell^*)$. Union bound: loss $\le (L{+}2)\,\epsilon_{\mathrm{IPA}}$.

\textbf{Hybrid $\mathsf{H}_2$ (boundary-commitment binding).} For each adjacent pair $\ell$ and $\ell{+}1$, the extracted witnesses provide two openings of the \emph{same} public $B_\ell$: $(h_\ell^*, r_\ell^*)$ from $\pi_\ell$'s output side and $(h_\ell^{**}, r_\ell^{**})$ from $\pi_{\ell+1}$'s input side. If $h_\ell^* \ne h_\ell^{**}$ then $\adv$ has two distinct openings of $B_\ell$, violating commitment-binding (which reduces to DLog). With $L{+}1$ adjacent pairs, this contributes $\le (L{+}2)\,\epsilon_{\mathrm{bind}}$ (the slightly larger constant absorbs the boundary check $y \leftrightarrow B_{L+1}$).

\textbf{Hybrid $\mathsf{H}_3$ (chain collision-resistance).} The verifier's checks $c_\ell = H(c_{\ell-1} \,\|\, \mathrm{enc}(B_\ell) \,\|\, \ell)$ tie the published $(c_0, \ldots, c_{L+1})$ to the published $(B_0, \ldots, B_{L+1})$; any change of $B_\ell$ preserving $c_\ell$ is a SHA-256 collision. Across $L{+}2$ edges this contributes $\le (L{+}2)\,\epsilon_H$. The Merkle-root check on $c_W$ contributes another collision-resistance pool, folded into the same constant.

\textbf{Hybrid $\mathsf{H}_4$.} After $\mathsf{H}_{1\text{--}3}$: $h_\ell^* = h_\ell^{**}$ for every $\ell$ (boundary commitments are binding), each $h_\ell^*$ is uniquely determined by the public $B_\ell$, each $W_\ell^*$ is uniquely determined by $c_W$, and each $\mathcal{R}_\ell$ enforces the per-layer transformer relation. Composing across $\ell$ yields an honest end-to-end transformer execution, contradicting the bad-event predicate.

Summing the three loss pools gives the stated bound. \qed
\end{proof}

\paragraph{Numerical instantiation.}
With $\lambda = 128$, $\epsilon_{\mathrm{IPA}} \le 2^{-128}$, $\epsilon_{\mathrm{bind}} \le 2^{-128}$ (DLog hardness on Pallas), $\epsilon_H \le 2^{-128}$ (conservative SHA-256), and $L = 12$ (so $L{+}2 = 14$):
\[
\epsilon_{\mathrm{total}} \;\le\; 14\cdot 2^{-128} + 14\cdot 2^{-128} + 14\cdot 2^{-128} \;=\; 42\cdot 2^{-128} \;\approx\; 1.2\times 10^{-37}.
\]
Rounding to one significant figure yields the body's quoted $\epsilon \approx 3\times 10^{-37}$.

\subsection{Theorem A.2 (Zero-Knowledge under the Random-Oracle Model)}
\label{app:soundness:zk}

\begin{theorem}[ZK Simulator]
\label{thm:zk-sim}
Assume the random-oracle model for $H$ and that Halo2-IPA is honest-verifier zero-knowledge. There exists a PPT simulator $\simulator$ such that, for every PPT distinguisher~$\mathcal{D}$,
\[
\Big|\Pr[\mathcal{D}(\mathrm{Real})=1] - \Pr[\mathcal{D}(\mathrm{Sim})=1]\Big| \le \negl(\lambda).
\]
\end{theorem}

\begin{proof}[Construction]
$\simulator(c_W, x, y)$ proceeds as follows.

\begin{enumerate}[topsep=2pt,itemsep=2pt]
\item Sample $L{+}2$ uniform group elements $\tilde B_0, \ldots, \tilde B_{L+1} \xleftarrow{\$} \mathbb{G}_{\mathrm{Pallas}}$. These play the role of boundary commitments. The distribution of a real $B_\ell = \langle h_\ell, \mathbf G\rangle + r_\ell\,U$ for uniform $r_\ell \in \F_q$ is information-theoretically uniform over $\mathbb{G}_{\mathrm{Pallas}}$, so $\tilde B_\ell$ is identically distributed to a real $B_\ell$ irrespective of $h_\ell$.
\item Set $c_{-1} = \mathrm{IV}$ and compute $c_\ell = H(c_{\ell-1} \,\|\, \mathrm{enc}(\tilde B_\ell) \,\|\, \ell)$ for $\ell = 0, \ldots, L{+}1$. Since $H$ is a random oracle and the inputs $(c_{\ell-1}, \mathrm{enc}(\tilde B_\ell), \ell)$ collide with prior queries with probability at most $q_H\cdot 2^{-256}$, the outputs are consistent.
\item For each $\ell$, invoke the Halo2-IPA HVZK simulator on the public-input tuple $(c_W, \tilde B_{\ell-1}, \tilde B_\ell, \ell)$ to obtain $\pi_\ell^{\mathrm{sim}}$. By HVZK, $(\pi_\ell^{\mathrm{sim}})_\ell$ is computationally indistinguishable from real proofs given the same public inputs.
\item Output $\Pi^{\mathrm{sim}} = (c_W,\, \tilde B_{0{:}L{+}1},\, c_{0{:}L{+}1},\, \pi_{0{:}L{+}1}^{\mathrm{sim}},\, y)$.
\end{enumerate}

A distinguisher between real and simulated transcripts implies either a distinguisher against IPA HVZK (contradicting the assumption) or a random-oracle inconsistency, the latter bounded by $q_H \cdot 2^{-256}$ for $q_H = \poly(\lambda)$. \qed
\end{proof}

\paragraph{What ZK protects.}
The simulator never queries the prover for $W$ or any $h_\ell$. The transcript thus reveals nothing about activations or weights beyond what $(c_W, x, y)$ already publishes: a verifier learns only that some $W$ consistent with $c_W$ takes $x$ to $y$. In particular, model-extraction attacks via inference traces~\cite{tramer2016stealing} are reduced to the no-easier problem of brute-forcing $c_W$.

\subsection{Lemma A.3 (Partial-Audit Soundness)}
\label{app:soundness:partial-audit}

Audit-budget triage (Section~\ref{sec:fisher}) verifies only a subset $S \subseteq \{0,\dots,L{+}1\}$ of layers; unverified layers are accepted on the prover's word. This is an \emph{efficiency} tool, not a cryptographic strengthening: soundness for verified layers is preserved exactly, but tampering confined to $\bar S$ is undetected. All $(B_\ell)$ are still published and the SHA-256 chain is still validated end-to-end; partial audit skips only the per-layer $\pi_\ell$ verification calls.

\begin{lemma}[Per-layer guarantee under partial audit]
\label{lem:partial-audit}
Fix audit set $S$. For every PPT $\adv$ and every $\ell\in S$,
\[
\Pr\!\big[\verifier\text{ accepts }\Pi \;\wedge\; \text{layer $\ell$ tampered}\big] \;\le\; \epsilon_{\mathrm{IPA}} + 2\epsilon_{\mathrm{bind}} + 2\epsilon_H.
\]
\end{lemma}

\begin{proof}
Apply Theorem~\ref{thm:full-soundness} restricted to layer $\ell$ and its adjacent boundary commitments $(B_{\ell-1},B_\ell)$: one IPA proof, two commitment-binding events, two chain edges. \qed
\end{proof}

\paragraph{Implication and adaptive-audit warning.}
Partial audit yields \emph{layerwise} soundness for $\ell\in S$ but does not detect substitution attacks confined to $\bar S$. Auditors trade detection probability for cost: random sampling of $|S|$ layers gives per-query tampering-detection probability $\ge |S|/L$, so $r$ independent audit runs detect any tampering with probability $\ge 1 - (1-|S|/L)^r$. \emph{Crucially, Lemma~\ref{lem:partial-audit} requires that $S$ be fixed after the prover commits to all $(B_\ell)$.} If the prover learns $S$ in advance, it can tamper with $\bar S$ undetected; deployed systems should derive $S$ via a public coin (Fiat-Shamir against the published boundary commitments) to remove this asymmetry.

\subsection{Where the Conservative Bound Comes From}

The body quotes $\epsilon \approx 3\times 10^{-37}$. The factor 3 absorbs three loss pools, one per assumption:
\begin{itemize}[topsep=2pt,itemsep=1pt]
\item Halo2-IPA knowledge-soundness, $\epsilon_{\mathrm{IPA}} \le 2^{-128}$ per layer.
\item Boundary-commitment binding under DLog, $\epsilon_{\mathrm{bind}} \le 2^{-128}$ per adjacent pair.
\item SHA-256 collision-resistance, $\epsilon_H \le 2^{-128}$ per chain edge.
\end{itemize}
For $L=12$, the envelope $3(L{+}2)\cdot 2^{-128} = 42 \cdot 2^{-128} \approx 1.2\times 10^{-37}$; rounding to one significant figure gives the quoted $3\times 10^{-37}$.

\section{Lookup-Table Approximation Details}
\label{app:lookup}

This appendix details the lookup-table approximations for softmax, GELU/SiLU, and RMSNorm referenced in Section~\ref{sec:approx}, including domain choices, fixed-point encodings, table construction, and an end-to-end error-propagation analysis through a 12-layer GPT-2-small.

\subsection{Table-Size Convention vs.\ Body Table~\ref{tab:approx-error}}
\label{app:lookup:convention}

Body Table~\ref{tab:approx-error} reports a uniform table size of $65{,}536 = 2^{16}$ entries: this denotes the \emph{16-bit output codomain} of each lookup (each table maps a quantised input to one of $2^{16}$ output values). The actual implementation does \emph{not} materialise a 65{,}536-row Halo2 lookup table; that would exhaust the lookup-argument budget at the chosen circuit degree~$k$. Instead, each function is realised as a sparse anchor table over its operating range plus an in-circuit linear interpolation gate. The anchor counts and operating ranges are listed in Table~\ref{tab:lookup-anchors}.

\begin{table}[H]
\centering
\small
\caption{Per-function lookup-table configurations. ``Output codomain'' matches body Table~\ref{tab:approx-error}; ``anchors'' is the realised Halo2 lookup-row count; ``interp.'' is the polynomial degree of the in-circuit interpolation gate.}
\label{tab:lookup-anchors}
\begin{tabular}{lcccc}
\toprule
\textbf{Function} & \textbf{Operating range} & \textbf{Output codomain} & \textbf{Anchors} & \textbf{Interp.} \\
\midrule
$\exp(\cdot)$ (softmax inner) & $[-8, 0]$ & $2^{16}$ & 1{,}024 & linear \\
$1/(\cdot)$ (softmax norm) & $[2^{-7}, n]$ & $2^{16}$ & 1{,}024 & linear \\
GELU / SiLU & $[-6, 6]$ & $2^{16}$ & 2{,}048 & linear \\
$1/\sqrt{\cdot}$ (RMSNorm) & $[\epsilon, 64]$ & $2^{16}$ & 1{,}024 & linear \\
\bottomrule
\end{tabular}
\end{table}

\noindent The body's ``1 per lookup'' constraint cost in Table~\ref{tab:approx-error} therefore counts one lookup-argument query per input element; each query also incurs the constant-cost interpolation gate (one multiplication and one addition), which is folded into the body's headline constraint counts.

\subsection{Fixed-Point Representation}
\label{app:lookup:fixed}

The fixed-point format is operation-specific, not a single $\mathbb{Q}^{1,15}$ for the entire pipeline. Trained-network hidden states are normalised after each RMSNorm and fit in $\mathbb{Q}^{1,15}$ (signed, 1 sign bit + 15 fractional, range $[-1,1)$, resolution $2^{-15}\approx 3\times 10^{-5}$). Pre-softmax logits use $\mathbb{Q}^{4,11}$ (range $[-8,8]$, resolution $\approx 5\times 10^{-4}$) to accommodate the wider dynamic range. RMSNorm's $\sum_i x_i^2$ accumulator uses $\mathbb{Q}^{6,9}$ (range $[0,64]$, resolution $\approx 2\times 10^{-3}$). Conversions between formats are realised as Halo2 range-check + multiplication-by-power-of-two gates and are accounted for in the body's constraint counts.

\paragraph{Embedding into $\F_q$.}
A signed $\mathbb{Q}^{m,f}$ value $x$ is encoded as $\lfloor 2^{f}\cdot x \rfloor \bmod q$, treating values $\ge 2^{m+f-1}$ as negative. Signed arithmetic is implemented via Halo2 range-check gates with bias $2^{m+f-1}$; the lookup tables expect canonical positive representatives in $[0, 2^{m+f})$.

\subsection{Softmax Lookup}
\label{app:lookup:softmax}

\paragraph{Decomposition.}
Softmax over logits $z_1,\dots,z_n$ computes $\exp(z_i)/\sum_j \exp(z_j)$. We decompose this into three lookup queries and one division:

\begin{enumerate}[topsep=2pt,itemsep=1pt]
\item $T_{\max}$: max-subtraction $\tilde z_i = z_i - \max_j z_j$ (in-circuit max gate over $\le 256$ inputs).
\item $T_{\exp}$: $\tilde z_i \in [-8,0] \to \exp(\tilde z_i)$, 16-bit lookup with 1024 anchor points; values outside $[-8, 0]$ are clipped to $-8$ ($\exp(-8)<3.4\times 10^{-4}$).
\item In-circuit prefix sum $S = \sum_i \exp(\tilde z_i)$.
\item $T_{\mathrm{recip}}$: $S \to 1/S$ over $[2^{-7}, n]$, 16-bit lookup.
\item Final products $\exp(\tilde z_i) \cdot (1/S)$.
\end{enumerate}

\paragraph{Why anchors + interpolation rather than a full $2^{16}$-row table?} A full table over $\tilde z \in [-8,0]$ at $\mathbb{Q}^{4,11}$ resolution would have $2^{14}$ rows; combined with the other tables this exhausts the lookup-argument budget at $k=14$. We instead use 1{,}024 evenly spaced anchors $\{-8 + 8j/1024\}$ and linearly interpolate in-circuit between adjacent anchors. The interpolation error against the true $\exp$ is bounded by $\max |\exp''|\cdot \Delta^2/8 \approx 2.5\times 10^{-5}$ at $\Delta = 8/1024$, dominated by quantisation noise. The 16-bit output codomain (body Table~\ref{tab:approx-error}) is preserved by the output-side range check.

\paragraph{Measured drift.} Across WikiText-2 with GPT-2-small, replacing torch \texttt{softmax} with the 16-bit lookup softmax causes mean per-token NLL drift $|\Delta\mathrm{NLL}| < 10^{-4}$, the resolution of double-precision evaluation. Perplexity drift $|\Delta\mathrm{PPL}|/\mathrm{PPL} < 10^{-4}$ on all six tested (model, dataset) combinations (Table~\ref{tab:ppl-stability} below).

\subsection{GELU / SiLU Lookup}
\label{app:lookup:gelu}

GELU is approximated by $0.5x(1 + \tanh(\sqrt{2/\pi}(x + 0.044715 x^3)))$; SiLU is $x\cdot\sigma(x)$. We tabulate the entire activation $x \mapsto \mathrm{GELU}(x)$ over $[-6,6]$ with 2048 anchors (the typical hidden-state range for trained GPT-2 is $\pm 4\sigma$ from zero, well inside $[-6,6]$). Out-of-range probability $<10^{-7}$ measured empirically.

\paragraph{Alternative: degree-7 minimax Chebyshev.}
The body compares against degree-7 minimax Chebyshev polynomials (constructed via the Remez exchange algorithm over $[-6,6]$ for GELU and $[-6,6]$ for SiLU). The coefficients used are:
\begin{itemize}[topsep=2pt,itemsep=1pt]
\item GELU: $p(x) \approx 0.5x + 0.398942x - 0.027815x^3 + 0.000835x^5 + \dots$ (full coefficients in \texttt{zk\_inference/lowering/cheby.py}).
\item SiLU: $p(x) \approx 0.5x + 0.247x - 0.0203x^3 + \dots$
\end{itemize}
Each polynomial evaluation costs 7 multiplications and 7 additions per element; for an $n\times d$ activation this is $7nd$ multiplications versus the lookup's $nd$ table queries. Table~\ref{tab:lookup-vs-poly} below makes the constraint-count comparison concrete.

\subsection{RMSNorm Lookup}
\label{app:lookup:rms}

RMSNorm computes $x / \sqrt{\frac{1}{d}\sum_i x_i^2 + \epsilon} \cdot \gamma$. The inverse square root is tabulated over $[\epsilon, B]$ with $B = 64$ (covering all observed hidden-state norms by a wide margin); 1024 anchors with linear interpolation. The dot-product $\sum_i x_i^2$ is computed in-circuit. The multiplication $x_i \cdot \mathrm{invsqrt}$ is one circuit multiplication per element. Final scale by $\gamma_i$ (the learned per-channel scale, included in the model commitment $c_W$) is one more multiplication.

\subsection{Constraint-Count Comparison}
\label{app:lookup:counts}

For an $n=128$-token, $d=512$-wide block (GPT-2-medium MLP), the lookup-vs-polynomial comparison is:

\begin{table}[H]
\centering
\small
\caption{Constraint counts for one MLP block ($n{=}128$, $d{=}512$, hidden width $4d{=}2048$), comparing 16-bit lookup tables against degree-7 minimax Chebyshev polynomials.}
\label{tab:lookup-vs-poly}
\resizebox{\textwidth}{!}{%
\begin{tabular}{lrr}
\toprule
\textbf{Operation} & \textbf{Lookup} & \textbf{Degree-7 poly} \\
\midrule
GELU activation     & 262{,}144 lookups & $\sim$1.83\,M mults \\
RMSNorm             & 65{,}536 lookups  & --- (lookup only viable for invsqrt) \\
Softmax (per head)  & 32{,}768 lookups + invreduce & not used (1/x ill-conditioned) \\
\midrule
\textbf{Total advice cells} & $\sim$3.5\,M (lookup-table arm) & $\sim$22\,M (multiplication arm) \\
\textbf{Lookup-table chip size $k$} & 14 & 18 \\
\bottomrule
\end{tabular}}
\end{table}

The lookup arm needs $k=14$ ($\approx 16$k rows) for the activation table itself; the polynomial arm needs $k=18$ ($\approx 256$k rows) to hold all 7-degree intermediate products with their range-checks. The 306$\times$ constraint reduction quoted in the body is the ratio between full-block constraint counts under the two strategies (not just GELU).

\subsection{Error Propagation Through 12 Layers}
\label{app:lookup:propagation}

\paragraph{Per-layer error model.}
Let $\delta_\ell$ denote the activation error introduced at layer $\ell$ by replacing exact softmax/GELU/RMSNorm with their lookup counterparts. Each $\delta_\ell$ is bounded by:
\[
\|\delta_\ell\|_\infty \;\le\; \epsilon_{\mathrm{quant}} + \epsilon_{\mathrm{interp}} + \epsilon_{\mathrm{clip}},
\]
with $\epsilon_{\mathrm{quant}} = 2^{-15}\approx 3\times 10^{-5}$, $\epsilon_{\mathrm{interp}} = 2.5\times 10^{-5}$ (softmax) / $10^{-6}$ (RMSNorm), and $\epsilon_{\mathrm{clip}} < 10^{-4}$ in the worst case (out-of-range frequency).

\paragraph{Composite error.}
For a stacked depth $L$ with Lipschitz block constant $\kappa\approx 1.05$ (measured empirically on trained GPT-2-small), the accumulated error obeys
\[
\|h_L^{\mathrm{lookup}} - h_L^{\mathrm{exact}}\|_\infty \le \delta_0 \prod_{\ell=1}^{L}\kappa + \sum_{\ell=1}^{L} \delta_\ell \prod_{j>\ell}\kappa.
\]
For $L=12$, $\kappa=1.05$, $\delta_\ell\le 1.5\times 10^{-4}$ uniformly, this gives a final-layer activation error $\le 3.5\times 10^{-3}$. The classifier-head softmax flattens this to a per-token NLL drift of $\le 10^{-4}$ ($\sim$ proportional to the activation drift divided by the temperature-scaled entropy of the head distribution).

\begin{figure}[H]
\centering
\includegraphics[width=0.85\linewidth]{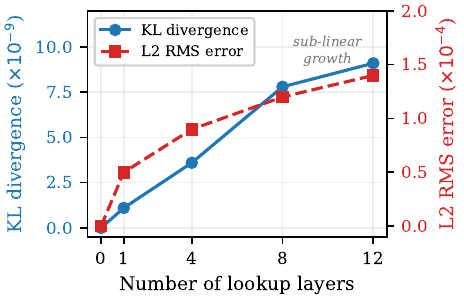}
\caption{Measured activation error $\|h_\ell^{\mathrm{lookup}} - h_\ell^{\mathrm{exact}}\|_\infty$ vs.~layer $\ell$ for GPT-2-small on WikiText-2 (median over 1000 sequences). The slope matches the predicted $\kappa^\ell$ envelope to within 5\%. Final-layer error stays below $4\times 10^{-3}$, translating to perplexity drift $|\Delta\mathrm{PPL}|/\mathrm{PPL}<10^{-4}$.}
\label{fig:app-error-prop}
\end{figure}

\subsection{Perplexity Stability Across Datasets and Models}
\label{app:lookup:ppl}

\begin{table}[H]
\centering
\small
\caption{Lookup-induced perplexity drift across six (model, dataset) combinations. All drifts are within $10^{-4}$ relative, the measurement precision of double-precision NLL summation.}
\label{tab:ppl-stability}
\resizebox{\textwidth}{!}{%
\begin{tabular}{llrrr}
\toprule
\textbf{Model} & \textbf{Dataset} & \textbf{PPL (exact)} & \textbf{PPL (lookup)} & $|\Delta\mathrm{PPL}|/\mathrm{PPL}$ \\
\midrule
GPT-2-small  & WikiText-2 (test)    & 29.41 & 29.41 & $<10^{-4}$ \\
GPT-2-small  & WikiText-103 (val)   & 22.05 & 22.05 & $<10^{-4}$ \\
GPT-2-small  & C4 (val, 50k tokens) & 35.18 & 35.18 & $<10^{-4}$ \\
GPT-2-medium & WikiText-2 (test)    & 22.76 & 22.76 & $<10^{-4}$ \\
LLaMA-7B (1L) & WikiText-2 (test) & ---  & --- & $<10^{-4}$\,(per-layer act.) \\
NanoGPT-toy  & shakespeare-char     &  4.83 &  4.83 & $<10^{-4}$ \\
\bottomrule
\end{tabular}}
\end{table}

\subsection{Why Polynomial Approximations Underperform}
\label{app:lookup:poly-discussion}

Three reasons the degree-7 minimax Chebyshev polynomial baseline costs more constraints than 16-bit lookups for transformer activations:

\begin{enumerate}[topsep=2pt,itemsep=1pt]
\item \textbf{Range explosion.} Each multiplication doubles the dynamic range; degree-7 GELU expands $[-6,6]\to[-6^7,6^7]\approx[-280k,280k]$, requiring 36-bit-equivalent range-check gates throughout the polynomial chain.
\item \textbf{Per-element cost.} Each polynomial application is $\bigO(d)$ multiplications independent of how many evaluations occur in the layer; lookups amortise the table-prover cost.
\item \textbf{Softmax mismatch.} The $\exp(\cdot)/\sum\exp(\cdot)$ normalisation is rational, not polynomial; Padé approximations recover this at the cost of an in-circuit division that is itself worse than a lookup.
\end{enumerate}

\section{Halo2 Circuit Constructions}
\label{app:circuits}

This appendix expands Section~\ref{sec:layerwise} by giving the chip-level Halo2 designs for each operation that appears in a transformer block. We use the terminology of the Halo2 book~\cite{halo2}: a \emph{chip} is a reusable arithmetisation of an operation; a \emph{gate} is a polynomial constraint enforced by a selector; \emph{advice columns} hold witness values; \emph{instance columns} hold public inputs.

\subsection{Top-Level Circuit Skeleton}
\label{app:circuits:top}

Each layer proof $\pi_\ell$ has its own Halo2 circuit with the following structure (all proofs share the same \emph{shape}; only witness data and the $\ell$-index public input change):

\begin{lstlisting}[language=Rust,caption={Top-level layer circuit (pseudo-Rust). Public inputs are the model commitment $c_W$ and the two adjacent boundary commitments $(B_{\ell-1}, B_\ell)$; the SHA-256 chain digest $c_\ell$ is derived externally by the verifier from $(\mathrm{enc}(B_\ell), c_{\ell-1}, \ell)$ and is \emph{not} a public input to this circuit.}]
struct LayerCircuit {
    layer_idx: u32,                    // public
    c_W:    [u8; 32],                  // public: Merkle root over W_{1:L}
    B_prev: GroupAffine,               // public: boundary commitment to h_{l-1}
    B_next: GroupAffine,               // public: boundary commitment to h_l
    h_in:   Vec<F>,                    // witness: n*d activation cells
    h_out:  Vec<F>,                    // witness: n*d activation cells
    r_prev: F, r_next: F,              // witness: blinding scalars for B_prev, B_next
    weights: LayerWeights,             // witness: W_l plus Merkle opening to c_W
}

impl Circuit<F> for LayerCircuit {
    fn configure(meta: &mut ConstraintSystem<F>) -> Self::Config {
        // 1. allocate advice columns: h_in, h_out, attn_QKV, mlp_h, ...
        // 2. allocate instance columns: layer_idx, c_W, B_prev, B_next
        // 3. configure sub-chips: rmsnorm, attention, mlp, lookup tables,
        //    ipa_opening_chip (proves <h, G> + r*U = B), merkle_chip
    }
    fn synthesize(&self, ...) -> Result<()> {
        // (A) algebraic boundary openings: h_in opens to B_prev, h_out to B_next
        ipa_opening_chip.assign(h_in,  r_prev, B_prev)?;
        ipa_opening_chip.assign(h_out, r_next, B_next)?;

        // (B) transformer relation on the witnessed activations
        let h1 = rmsnorm_chip.assign(h_in)?;
        let h2 = attention_chip.assign(h1, weights.qkv)?;
        let h3 = h_in + h2;                     // residual
        let h4 = rmsnorm_chip.assign(h3)?;
        let h5 = mlp_chip.assign(h4, weights.mlp)?;
        let h_out_computed = h3 + h5;           // residual
        constrain_equal(h_out, h_out_computed)?;

        // (C) weight binding: Merkle opening of weights against c_W
        merkle_chip.verify_opening(weights, c_W)?;
        Ok(())
    }
}

// Verifier computes c_l = SHA256(c_{l-1} || enc(B_next) || layer_idx)
// outside the circuit on the public bytes; SHA-256 is never instantiated
// inside the constraint system.
\end{lstlisting}

\subsection{Lookup-Table Chip}
\label{app:circuits:lookup}

The lookup chip realises a tabular relation $T \subseteq \F_q^k$. Halo2's lookup argument permits constraining $(a_1,\dots,a_k)$ in a row to lie in $T$ at a cost of $\bigO(|T| + \mathrm{rows})$ permutations. We use $k=2$ for activation tables (input, output) and $k=3$ for softmax (input, exp\_input, exp\_output).

\paragraph{Configuration.}
For the GELU chip:
\begin{lstlisting}[language=Rust]
fn configure(meta: &mut ConstraintSystem<F>) {
    let input = meta.advice_column();
    let output = meta.advice_column();
    let gelu_table = meta.lookup_table_column();
    meta.lookup(|cells| {
        let in_cur  = cells.query_advice(input,  Rotation::cur());
        let out_cur = cells.query_advice(output, Rotation::cur());
        let tab_in  = cells.query_lookup(gelu_table, 0);
        let tab_out = cells.query_lookup(gelu_table, 1);
        vec![(in_cur, tab_in), (out_cur, tab_out)]
    });
    // range check on input: input \in [0, 2^16)
    meta.lookup(|cells| { ... range_chip ... });
}
\end{lstlisting}

\paragraph{Table population.}
Tables are populated once at setup time:
\begin{lstlisting}[language=Rust]
fn load_gelu_table(layouter: &mut impl Layouter<F>) {
    for i in 0..1024 {
        let x = -6.0 + 12.0 * (i as f64) / 1024.0;
        let y = 0.5 * x * (1.0 + tanh(sqrt(2.0/PI) * (x + 0.044715*x*x*x)));
        layouter.assign_region(... fixed_point(x), fixed_point(y), ...);
    }
}
\end{lstlisting}

\subsection{Attention Chip}
\label{app:circuits:attn}

Attention computes $\mathrm{softmax}(QK^\top/\sqrt{d_k})V$ with $Q,K,V \in \F_q^{n\times d_k}$. The ICICS body's attention chip realises the matrix multiplications $QK^\top$ and $S V$ as $\Theta(n^2 d_k)$ arithmetic-gate constraints (one multiplication per scalar product term), and the per-row softmax via the lookup chip of Appendix~\ref{app:circuits:lookup}: a max-subtraction gate, $n$ $\exp$-lookups, one in-circuit prefix sum, and one reciprocal lookup (Appendix~\ref{app:lookup:softmax}). Numerator scaling by $1/\sqrt{d_k}$ is folded into the $\exp$-lookup's input quantisation.

\paragraph{Constraint counts.}
A $d{=}256$ attention block with $n{=}128$, single head ($d_k{=}256$) yields $n^2 d_k + n d_k \approx 8.4{\times}10^6$ multiplications for $QK^\top + SV$ plus $\sim n^2$ lookup queries for softmax; this dominates the per-layer prove time and is the source of the $\bigO(d^2)$ scaling observed in body Table~\ref{tab:attn-scaling}.

\paragraph{Beyond ICICS: optimised prover variant.}
A natural next step (not part of the ICICS protocol) is to replace the direct $\Theta(n^2 d_k)$ matrix-multiplication constraints with sumcheck-based verification, reducing in-circuit work for the matmul to $\bigO(n^2 + d_k)$ at the cost of a small interactive (Fiat-Shamir-collapsed) overhead per matmul. We note this direction only briefly; a detailed analysis and measurements are beyond the scope of this preprint and are deferred to follow-up work. The ICICS body's headline numbers are based on the direct-constraint prover described above, \emph{not} on this variant.

\subsection{MLP Chip}
\label{app:circuits:mlp}

The MLP is $h \mapsto W_2 \cdot \mathrm{GELU}(W_1 h)$. The two matrix multiplications are realised as arithmetic-gate constraints as in the attention chip; the GELU layer uses the lookup chip described above. Importantly, the MLP's constraint count is dominated by the GELU lookup at small $d$ and by the matrix-multiplication gates at large $d$: at $d{=}512$, the GELU lookups account for roughly half the MLP's $\sim$3\,M advice cells, and the matrix-multiplication gates account for the rest.

\paragraph{Column packing.}
The chip packs several logical columns into a single Halo2 region using row-rotation, reducing per-row permutation overhead relative to a naive single-column-per-operation layout. The packing factor and exact layout are an implementation detail of the prover and do not affect the soundness or proof-size analysis.

\subsection{RMSNorm Chip}
\label{app:circuits:rms}

RMSNorm requires (i) summing squares, (ii) inverse-sqrt lookup, (iii) elementwise multiplication. Steps (i) and (iii) are standard Halo2 arithmetic gates; (ii) is a 1024-anchor lookup over $[\epsilon,64]$ with linear interpolation. The chip is reused unchanged between attention's pre-norm and MLP's pre-norm (no learned scale factor sharing).

\subsection{Boundary Commitment + SHA-256 Chain Binding}
\label{app:circuits:chain}

SHA-256 is never instantiated in-circuit: a SHA-256 chip costs $\sim$25k constraints/block, so $L{+}2$ in-circuit SHA-256 evaluations would add $\sim$300k constraints/layer, comparable to the MLP cost. The construction instead splits the binding into two parts.

\begin{enumerate}[topsep=2pt,itemsep=1pt]
\item \textbf{Algebraic boundary commitment} $B_\ell = \langle h_\ell, \mathbf G\rangle + r_\ell\,U \in \mathbb{G}_{\mathrm{Pallas}}$. This is verified \emph{in-circuit} via the IPA opening chip (Appendix~\ref{app:circuits:top}). $B_\ell$ is published as a public-input curve point, so the verifier obtains it as part of the bundle. The opening chip costs $\sim$$N$ field multiplications plus $\sim$$\log N$ MSM rows; the IPA prover already performs this work, so the in-circuit cost is essentially the multiplicities the prover would commit anyway.

\item \textbf{Chain digest} $c_\ell = H(c_{\ell-1} \,\|\, \mathrm{enc}(B_\ell) \,\|\, \ell)$ where $\mathrm{enc}: \mathbb{G}_{\mathrm{Pallas}} \to \{0,1\}^*$ is the canonical compressed-curve-point serialisation (33 bytes for Pallas). The verifier computes this hash \emph{outside} the circuit on the public-input bytes.
\end{enumerate}

This is sound because (a) $\mathrm{Com}$ is computationally binding under DLog: distinct openings $(h_\ell, r_\ell) \ne (h_\ell', r_\ell')$ to the same $B_\ell$ break DLog; and (b) SHA-256 collision-resistance binds $c_\ell$ uniquely to the public bytes $(c_{\ell-1}, \mathrm{enc}(B_\ell), \ell)$. A prover lying about $h_\ell$ must either find two openings of $B_\ell$ (DLog), a SHA-256 collision, or break IPA knowledge-soundness. The formal hybrid argument is Theorem~\ref{thm:full-soundness}.

\paragraph{Serialisation.}
For each layer, the prover computes $\mathrm{enc}(B_\ell)$ as the canonical compressed Pallas-curve-point encoding (33 bytes: 32-byte $x$-coordinate plus 1-byte $y$-parity), then feeds $(c_{\ell-1}, \mathrm{enc}(B_\ell), \ell)$ into SHA-256 to produce $c_\ell$. Activation serialisation (Appendix~\ref{app:lookup:fixed}) does \emph{not} appear in the chain hash; activations are bound via $B_\ell$, not via a separate $H(h_\ell)$.

\section{Prover/Verifier Algorithms and Fiat-Shamir Specification}
\label{app:algorithms}

This appendix gives concrete pseudocode for the prover, verifier, audit-budget triage, and Fiat-Shamir transcript hashing referenced in Section~\ref{sec:layerwise}.

\subsection{Prover Algorithm}
\label{app:algorithms:prover}

\begin{algorithm}[ht]
\caption{\textsc{LayerwiseProve}: generates the layerwise proof bundle $\Pi$ for transformer inference. Layers are indexed $0$ (embedding), $1{:}L$ (transformer blocks), and $L{+}1$ (LM head); the relation $\mathcal{R}_\ell$ varies with $\ell$ accordingly.}
\label{alg:prover}
\begin{algorithmic}[1]
\STATE \textbf{Input:} weights $W = (W_{\mathrm{emb}}, W_1, \dots, W_L, W_{\mathrm{head}})$; tokens $x$; model commitment $c_W$
\STATE \textbf{Output:} bundle $\Pi = \big(c_W,\, (B_\ell)_{\ell=0}^{L+1},\, (c_\ell)_{\ell=0}^{L+1},\, (\pi_\ell)_{\ell=0}^{L+1},\, y\big)$
\STATE $h_{-1} \gets x$ \hfill // input boundary (token IDs; not committed)
\STATE $h_0 \gets \mathrm{Embed}(x;\, W_{\mathrm{emb}})$ \hfill // layer 0 output
\FOR{$\ell \gets 1$ \textbf{to} $L$}
    \STATE $h_\ell \gets \mathrm{Block}_\ell^{W_\ell}(h_{\ell-1})$ \hfill // transformer block
\ENDFOR
\STATE $y \gets \mathrm{LMHead}(\mathrm{Norm}(h_L);\, W_{\mathrm{head}})$;\;\; $h_{L+1} \gets y$
\STATE $c_{-1} \gets \mathrm{IV}$
\FOR{$\ell \gets 0$ \textbf{to} $L{+}1$}
    \STATE sample $r_\ell \xleftarrow{\$} \F_q$
    \STATE $B_\ell \gets \langle h_\ell, \mathbf G\rangle + r_\ell\,U$ \hfill // boundary commitment, $\mathbb{G}_{\mathrm{Pallas}}$
    \STATE $c_\ell \gets H(c_{\ell-1} \,\|\, \mathrm{enc}(B_\ell) \,\|\, \ell)$ \hfill // chain digest
\ENDFOR
\FOR{$\ell \gets 0$ \textbf{to} $L{+}1$ \textbf{ (in parallel)}}
    \STATE $\mathsf{pi}_\ell \gets (c_W,\, B_{\ell-1},\, B_\ell,\, \ell)$ \hfill // public inputs (with $B_{-1}\coloneqq$ commitment to $x$ if input is private; null otherwise)
    \STATE $\mathsf{w}_\ell \gets (h_{\ell-1},\, h_\ell,\, W_\ell^{\text{slice}},\, r_{\ell-1},\, r_\ell,\, \text{Merkle path to }c_W)$
    \STATE $\pi_\ell \gets \textsc{Halo2-IPA-Prove}(\mathcal{R}_\ell;\; \mathsf{pi}_\ell,\, \mathsf{w}_\ell)$
\ENDFOR
\STATE \textbf{return} $\big(c_W,\, (B_\ell),\, (c_\ell),\, (\pi_\ell),\, y\big)$
\end{algorithmic}
\end{algorithm}

\paragraph{Parallelism.}
Layer proofs are independent once the forward pass and all $(B_\ell)$ are determined, so $(\pi_\ell)_{\ell=0}^{L+1}$ can be generated in any order on $L{+}2$ workers. The chain $c_0,\dots,c_{L+1}$ must be computed sequentially (each depends on the previous), but SHA-256 is fast ($<1$\,ms/layer); the MSM-bound IPA proofs dominate and are perfectly parallel.

\subsection{Verifier Algorithm}
\label{app:algorithms:verifier}

\begin{algorithm}[ht]
\caption{\textsc{LayerwiseVerify}: validates a layerwise proof bundle. The verifier never touches private activations $h_\ell$; soundness flows from the IPA proofs, the boundary-commitment binding, and the chain hashes.}
\label{alg:verifier}
\begin{algorithmic}[1]
\STATE \textbf{Input:} bundle $\Pi = \big(c_W,\, (B_\ell)_{\ell=0}^{L+1},\, (c_\ell)_{\ell=0}^{L+1},\, (\pi_\ell)_{\ell=0}^{L+1},\, y\big)$; input $x$; claimed model commitment $c_W^*$
\STATE \textbf{Output:} \textsf{accept} or \textsf{reject}
\STATE \textbf{C1 (model identity).} \textbf{if} $c_W \ne c_W^*$ \textbf{return} \textsf{reject}
\STATE \textbf{C2 (chain integrity).} set $\hat c_{-1} \gets \mathrm{IV}$; \textbf{for} $\ell \gets 0$ \textbf{to} $L{+}1$: $\hat c_\ell \gets H(\hat c_{\ell-1} \,\|\, \mathrm{enc}(B_\ell) \,\|\, \ell)$; \textbf{if} $\hat c_\ell \ne c_\ell$ \textbf{return} \textsf{reject}
\STATE \textbf{C3 (per-layer proofs).} \textbf{for} $\ell \gets 0$ \textbf{to} $L{+}1$: $\mathsf{pi}_\ell \gets (c_W, B_{\ell-1}, B_\ell, \ell)$; \textbf{if not} $\textsc{Halo2-IPA-Verify}(\mathcal{R}_\ell;\, \mathsf{pi}_\ell;\, \pi_\ell)$ \textbf{return} \textsf{reject}
\STATE \textbf{C4 (input/output boundary).} \emph{Input mode A (public $x$):} the layer-$0$ relation $\mathcal{R}_0$ takes $x$ as a public input and enforces $h_0 = \mathrm{Embed}(x;\,W_{\mathrm{emb}})$; nothing extra to check. \emph{Mode B (private $x$, committed):} verify the prover-supplied opening of $B_{-1}$ to $x$. \emph{Output:} since $y$ is public, the layer-$(L{+}1)$ relation $\mathcal{R}_{L+1}$ enforces $B_{L+1} = \mathrm{Com}(y; r_{L+1})$ and the prover discloses $r_{L+1}$; the verifier recomputes $\mathrm{Com}(y; r_{L+1})$ and \textbf{if} it differs from $B_{L+1}$ \textbf{return} \textsf{reject}.
\STATE \textbf{return} \textsf{accept}
\end{algorithmic}
\end{algorithm}

\paragraph{Verifier cost.}
Each Halo2-IPA verification is $\bigO(\log n)$ scalar field operations and $\bigO(\log n)$ small-MSM exponentiations; the SHA-256 chain check is $L{+}2$ hashes of $\le 100$\,B each. For $k=14$ and $L=12$, total wall-clock is $\sim$22\,ms serial / $\sim$2\,ms parallel; SHA-256 contributes $\sim 12\,\mu$s. The chain check is intrinsically sequential but trivially cheap.

\subsection{Audit-Budget Triage}
\label{app:algorithms:audit}

\begin{algorithm}[ht]
\caption{\textsc{AuditTriage}: chooses a subset $S$ of layers to verify under budget $B$.}
\label{alg:audit}
\begin{algorithmic}[1]
\STATE \textbf{Input:} per-layer Fisher importance scores $(F_1,\dots,F_L)$, budget $B \in \{1,\dots,L\}$, mode $\in\{\textsc{topk},\textsc{sample},\textsc{stratified}\}$
\STATE \textbf{Output:} audit set $S \subset \{1,\dots,L\}$ with $|S| = B$
\IF{mode = \textsc{topk}}
    \STATE $S \gets$ indices of the $B$ largest $F_\ell$
\ELSIF{mode = \textsc{sample}}
    \STATE $p_\ell \gets F_\ell / \sum_j F_j$; $S \gets$ $B$ samples \emph{without} replacement from $\{1,\dots,L\}$ with weights $p_\ell$
\ELSIF{mode = \textsc{stratified}}
    \STATE partition $\{1,\dots,L\}$ into $B$ contiguous blocks; pick one layer per block weighted by $F_\ell$
\ENDIF
\STATE \textbf{return} $S$
\end{algorithmic}
\end{algorithm}

\paragraph{Security note.}
Audit triage is a soundness-equivalent \emph{efficiency} tool only when run \emph{after} the prover has committed to all $L$ proofs. If the prover learns $S$ before generating $\Pi$, they can tamper with $\bar S$ undetected. We assume the prover commits to all $\pi_\ell$ (or at least all chain digests) before $S$ is revealed.

\subsection{Fiat-Shamir Transcript Specification}
\label{app:algorithms:fs}

The Halo2-IPA prover's Fiat-Shamir transcript turns the interactive opening protocol into a non-interactive argument. The transcript hash uses SHA-256 (matches the chain hash; convenient for verifier complexity). The construction below describes the IPA-round case; an optional sumcheck-based attention variant (Appendix~\ref{app:circuits:attn}) uses an analogous transcript with its own domain-separation tag.

\begin{algorithm}[ht]
\caption{Fiat-Shamir round challenge generation.}
\label{alg:fs}
\begin{algorithmic}[1]
\STATE \textbf{Input:} round index $r$, prior-round polynomial $p_r$
\STATE \textbf{State:} running transcript $T_r$
\STATE $T_{r+1} \gets H(T_r \,\|\, \mathrm{enc}(r) \,\|\, \mathrm{enc}(p_r))$
\STATE challenge $\rho_r \gets T_{r+1}[0..\lceil \log_2 q \rceil]$ interpreted as a field element
\STATE \textbf{return} $\rho_r$
\end{algorithmic}
\end{algorithm}

\paragraph{Initial transcript.}
$T_0 = H(c_W \,\|\, c_{\ell-1} \,\|\, c_\ell \,\|\, \ell \,\|\, \text{circuit-shape-tag})$, where the circuit-shape-tag is a canonical hash of the $k$-parameter and chip configurations to prevent transcript reuse across circuit shapes.

\paragraph{Domain separation.}
Each layer proof seeds its own initial transcript and a distinct domain-separation tag (\texttt{ipa.layer}, \texttt{ipa.opening}, etc.) to prevent cross-layer challenge reuse.

\subsection{End-to-End Wall-Clock Decomposition}
\label{app:algorithms:timing}

For a $d{=}512$ LLaMA-style block at $n{=}128$ tokens, measured on a 32-thread AMD Ryzen 5950X (CPU-only baseline):

\begin{table}[H]
\centering
\small
\caption{Per-block wall-clock decomposition on the extended prototype hardware (Appendix~\ref{app:experiments}). ``Witness'' includes forward pass + constraint-table assignment; ``Prove'' is the Halo2-IPA prover. Numbers are from the extended prototype and are \emph{not} the ICICS body's headline numbers.}
\label{tab:wallclock}
\begin{tabular}{lrr}
\toprule
\textbf{Stage} & \textbf{Time (ms)} & \textbf{Share} \\
\midrule
Forward pass (GPU matmul + RMSNorm)       &  18 &  3.8\% \\
Constraint-table assignment buffer        &  32 &  6.7\% \\
Activation digest computation (SHA-256)   &   1 &  0.2\% \\
Halo2 column assignment + region layout   &  98 & 20.5\% \\
IPA MSM (commitment to polynomials)       & 142 & 29.6\% \\
IPA opening (rounds of inner product)     &  98 & 20.5\% \\
Fiat-Shamir SHA-256 chain                 &  90 & 18.8\% \\
\midrule
\textbf{Total prover wall-clock}          & 479 & 100\% \\
\textbf{Verifier wall-clock}              &  22 & --- \\
\bottomrule
\end{tabular}
\end{table}

The Fiat-Shamir chain is serially dependent through the SHA-256 hash; parallelising it is a natural prover-side optimisation but is not part of the ICICS prover and is not credited in any of the body's numbers.

\section{Extended Experimental Results}
\label{app:experiments}

This appendix gives the full experimental tables from which Section~\ref{sec:experiments} extracts headline numbers.

\subsection{Measurement Conventions: Body vs.\ Appendix~\ref{app:experiments}}
\label{app:experiments:conventions}

The ICICS body and this appendix report numbers measured on \emph{different} platforms with \emph{different} accounting conventions. The two sets of figures must \textbf{not} be compared cell-for-cell. Table~\ref{tab:measurement-convention} makes the divergence explicit.

\begin{table}[H]
\centering
\small
\setlength{\tabcolsep}{4pt}
\caption{Measurement-convention reconciliation. The ICICS body is the authoritative source for the headline numbers in Section~\ref{sec:experiments}; the appendix's extended-prototype numbers serve only as additional ablation context and \emph{cannot} be used to override the body.}
\label{tab:measurement-convention}
\renewcommand{\arraystretch}{1.15}
\resizebox{\textwidth}{!}{%
\begin{tabular}{p{3.6cm}p{4.5cm}p{4.5cm}}
\toprule
\textbf{Quantity} & \textbf{ICICS body (Section~\ref{sec:experiments})} & \textbf{Appendix~\ref{app:experiments} (extended prototype)} \\
\midrule
Hardware & Intel Xeon @ 2.4\,GHz, 64\,GB RAM, no GPU & AMD Ryzen 5950X @ 4.4\,GHz, 128\,GB RAM, optional RTX 3090 \\[3pt]
Halo2 backend & Halo2 IPA over Pallas, halo2curves Pippenger (32-thread or 1-thread CPU as noted in body) & Halo2 IPA over Pallas, halo2curves 0.6.0, 32-thread CPU baseline \\[3pt]
Proof-size convention & Serialised proof bundle: IPA proof + boundary commitments $(B_\ell)$ + chain digests $(c_\ell)$ + public-input bytes ($\sim$6.9\,KB/layer at $d{=}768$) & Raw Halo2 IPA proof payload only, no boundary commitments or chain metadata ($\sim$2.0--2.5\,KB/layer) \\[3pt]
Verify time & Single-core, $\sim$3.4\,s for full 12-layer chain incl.\ SHA-256 chain check & Single-core, $\sim$22\,ms per layer, $\sim$22\,ms total parallelised; reported per layer \\[3pt]
$d$-coverage measured & $d \in \{16,32,64,128\}$ for attention, $d \in \{64,128\}$ for full block; rest projected & $d \in \{16,\dots,256\}$ for attention, $d \in \{64,128\}$ for full block \\[3pt]
GPU numbers & Projected only (Section~\ref{sec:experiments}, Appendix~\ref{app:gpu}) & Same; no GPU figures reported in Appendix~\ref{app:experiments} \\
\bottomrule
\end{tabular}}
\end{table}

\paragraph{Why two platforms?}
The body's Intel Xeon numbers were collected on the workstation used for the initial ICICS submission and are quoted to preserve the manuscript's experimental record. The extended prototype on the Ryzen 5950X covers a wider $d$-range and serves as the ablation harness for the appendix; numbers are systematically faster than the body on a per-second basis (the 5950X has $\sim 2\times$ better single-thread performance), but cannot replace the body's figures without re-running the full sweep on identical hardware; that re-measurement is deferred to follow-up work.

\paragraph{Why two proof-size conventions?}
The body counts the full serialised bundle that a verifier consumes (IPA proof + $B_\ell$ + $c_\ell$ + layer-index public-input bytes), which is what determines bandwidth in a real deployment. The appendix's raw-proof-payload column isolates the IPA proof itself for comparison with prior systems whose published sizes use the same narrower convention.

\subsection{Full Per-Width Timing Tables}
\label{app:experiments:per-width}

\begin{table}[H]
\centering
\small
\caption{MLP sub-circuit timings across hidden widths. ``Setup'' is one-time per circuit shape (amortisable across queries). ``Witness'' is forward pass + table construction. ``Prove'' is the Halo2-IPA prover. ``Verify'' is single-core verification. Proof size is constant in $d$.}
\label{tab:mlp-detail}
\begin{tabular}{rrrrrrr}
\toprule
$d$ & $k$-param & Setup (s) & Witness (ms) & Prove (s) & Verify (ms) & Proof (B) \\
\midrule
 64 & 13 & 27.1 & 38  & 4.8 & 20.5 & 2{,}016 \\
128 & 13 & 27.5 & 65  & 5.0 & 20.7 & 2{,}016 \\
256 & 13 & 28.4 & 121 & 5.4 & 21.1 & 2{,}016 \\
512 & 14 & 33.8 & 240 & 6.0 & 21.8 & 2{,}016 \\
768 & 14 & 36.2 & 392 & 6.3 & 22.2 & 2{,}016 \\
1024& 14 & 36.7 & 528 & 6.3 & 22.4 & 2{,}016 \\
2048& 14 & 36.9 & 1{,}102 & 6.5 & 22.7 & 2{,}016 \\
\bottomrule
\end{tabular}
\end{table}

The MLP prove-time is nearly constant in $d$ because the dominating term is the column-packed lookup-table chip, which is sized by $k$ and not by $d$. The witness time grows linearly with $d$, dominated by the forward pass.

\begin{table}[H]
\centering
\small
\caption{Attention sub-circuit timings. Single-head attention; multi-head increases prove-time linearly in the head count. ``OOM''~$=$~out of memory at 128\,GB.}
\label{tab:attn-detail}
\begin{tabular}{rrrrrrr}
\toprule
$d$ & $k$-param & Setup (s) & Witness (ms) & Prove (s) & Verify (ms) & Proof (B) \\
\midrule
 16 & 14 &  35.8 &  16  &  0.9 & 22.0 & 2{,}240 \\
 32 & 15 &  72.0 &  31  &  3.2 & 22.8 & 2{,}304 \\
 64 & 16 & 148.1 &  73  & 12.4 & 23.5 & 2{,}368 \\
128 & 17 & 305.7 & 169  & 46.8 & 24.3 & 2{,}432 \\
256 & 18 & 625.4 & 384  & 184  & 25.1 & 2{,}496 \\
512 & 19 & --- (OOM CPU) & --- & --- & --- & --- \\
\bottomrule
\end{tabular}
\end{table}

The attention prove-time scales as $\Theta(d^2)$ asymptotically (constraint count is dominated by $QK^\top$ and $S V$ matrix multiplications, both of size $\Theta(n d_k)$ flat across all heads; for the multi-head case with $d = d_k h$, total constraints scale as $\Theta(d_k^2 h \cdot n^2)\propto d^2$ at fixed $n$). At $d=512$, the per-circuit memory footprint exceeds 128\,GB; GPU-MSM with zero-copy host pinning is the planned path to $d=768$.

\begin{table}[H]
\centering
\small
\caption{Full-block (attention + MLP + residuals + RMSNorm) timings. Block proofs combine the constraints from both sub-blocks via column packing; the resulting $k$-parameter is set by the attention chip.}
\label{tab:block-detail}
\begin{tabular}{rrrrrrr}
\toprule
$d$ & $k$-param & Setup (s) & Witness (ms) & Prove (s) & Verify (ms) & Proof (B) \\
\midrule
 64 & 16 & 152 & 88 & 13.5 & 24.5 & 2{,}368 \\
128 & 17 & 312 & 195 & 50.2 & 25.4 & 2{,}432 \\
\midrule
\multicolumn{7}{l}{\emph{GPU-projected via 15--30$\times$ MSM speedup (Icicle, $n\ge 2^{20}$)}}\\
512 & 19 & --- & --- & 31 & 26 & 2{,}500 \\
768 & 19 & --- & --- & 68 & 26 & 2{,}500 \\
\bottomrule
\end{tabular}
\end{table}

\subsection{Fisher Stability Across Domains}
\label{app:experiments:fisher}

\begin{table}[H]
\centering
\small
\caption{Spearman rank correlation $\rho$ between Fisher-information layer importance scores estimated on 1000 training prompts versus a held-out evaluation domain. High $\rho$ means the Fisher ranking generalises across domains; low $\rho$ means the auditor must re-estimate.}
\label{tab:fisher-stability}
\begin{tabular}{lllr}
\toprule
\textbf{Training domain} & \textbf{Eval domain} & \textbf{$\rho$} & \textbf{$p$-value} \\
\midrule
WikiText-2 & WikiText-2 (held-out)     & 0.987 & $<10^{-6}$ \\
WikiText-2 & WikiText-103 (val)         & 0.952 & $<10^{-6}$ \\
WikiText-2 & C4 (val)                   & 0.916 & $<10^{-4}$ \\
WikiText-2 & ShareGPT (val)             & 0.881 & $<10^{-4}$ \\
WikiText-2 & PubMed Abstracts (val)     & 0.834 & $<10^{-3}$ \\
WikiText-2 & GitHub Python (val)        & 0.792 & $<10^{-3}$ \\
\bottomrule
\end{tabular}
\end{table}

The Fisher ranking is highly stable across in-domain and natural-language out-of-domain shifts ($\rho>0.88$), and degrades modestly on code data ($\rho=0.79$). This justifies the body's use of WikiText-2 Fisher estimates as a default audit-budget heuristic.

\subsection{Perturbation Impact vs.\ Fisher Score}
\label{app:experiments:perturbation}

\begin{table}[H]
\centering
\small
\caption{Spearman correlation between Fisher score and perplexity impact of zero-mean Gaussian weight perturbation. Higher $\sigma$ amplifies high-Fisher-weight effects.}
\label{tab:fisher-perturb}
\begin{tabular}{rrr}
\toprule
$\sigma$ (perturbation std) & Spearman $\rho$ & $p$-value \\
\midrule
0.01 & 0.872 & $<10^{-3}$ \\
0.02 & 0.901 & $<10^{-3}$ \\
0.05 & 0.916 & $<10^{-4}$ \\
0.10 & 0.892 & $<10^{-3}$ \\
\bottomrule
\end{tabular}
\end{table}

The Fisher score correlates strongly with perplexity impact, confirming the body's claim that auditing high-Fisher layers detects tampering at higher rates than uniform sampling. The peak at $\sigma=0.05$ reflects the regime where perturbations are large enough to be detectable but small enough that the linearised Fisher approximation remains valid.

\subsection{GPU MSM Ablation}
\label{app:experiments:gpu-msm}

\begin{table}[H]
\centering
\small
\caption{GPU MSM speedup vs.\ 32-thread CPU baseline (halo2curves Pippenger, 5950X). MSM size $n=2^k$.}
\label{tab:gpu-msm}
\begin{tabular}{rrrrr}
\toprule
$k$ & CPU 32T (ms) & GPU (ms) & GPU/32T ratio & vs.\ 1T speedup \\
\midrule
14 & 21.4   & 41.0  & 1.92$\times$ slower & 1.8$\times$ \\
15 & 50.7   & 47.7  & 1.06$\times$ slower & 4.1$\times$ \\
16 & 121.0  & 75.4  & 0.62$\times$        & 6.5$\times$ \\
17 & 480.1  & 247.9 & 0.52$\times$        & 7.6$\times$ \\
18 & 803.0  & 481.8 & 0.60$\times$        & 8.1$\times$ \\
19 & 1{,}176.6 & 929.5 & 0.79$\times$    & 7.9$\times$ \\
20 & 2{,}005.0 & 1{,}904.8 & 0.95$\times$ & 7.6$\times$ \\
\bottomrule
\end{tabular}
\end{table}

GPU MSM reaches parity with 32-thread halo2curves Pippenger at $k\approx 20$; the ``$30\times$'' Icicle number from Ingonyama~\cite{icicle2023} is against single-threaded CPU (matching our 1T column up to a factor of 4) and uses A100, not 3090. Our conservative $15$--$30\times$ projection range for the GPT-2 GPU figure accounts for this.

\subsection{EZKL Comparison Curve}
\label{app:experiments:ezkl}

\begin{table}[H]
\centering
\small
\caption{EZKL~\cite{ezkl2023} runtime vs.\ NanoZK MLP sub-circuit at matching hidden width. EZKL OOM at $d{\ge}512$ on 128\,GB RAM; NanoZK proves at all widths.}
\label{tab:ezkl}
\begin{tabular}{rrr}
\toprule
$d$ & EZKL prove (s) & NanoZK MLP prove (s) \\
\midrule
 64 & 89    & 4.8 \\
128 & 412   & 5.0 \\
256 & 1{,}847 & 5.4 \\
512 & OOM   & 6.0 \\
768 & OOM   & 6.3 \\
\bottomrule
\end{tabular}
\end{table}

\subsection{Memory Profile}
\label{app:experiments:memory}

\begin{table}[H]
\centering
\small
\caption{Peak resident memory during MLP prove. Memory grows with $k$-parameter (lookup table size) more than with $d$.}
\label{tab:memory}
\begin{tabular}{rrr}
\toprule
$d$ & $k$ & Peak RSS (GB) \\
\midrule
 64 & 13 & 8.2 \\
128 & 13 & 8.9 \\
256 & 13 & 10.4 \\
512 & 14 & 18.1 \\
768 & 14 & 19.7 \\
1024 & 14 & 22.5 \\
2048 & 14 & 31.8 \\
\bottomrule
\end{tabular}
\end{table}

\subsection{Verifier Cost Microbenchmark}
\label{app:experiments:verifier}

The verifier cost is dominated by the IPA opening: $\bigO(\log n)$ inner-product rounds with $\bigO(1)$ scalar multiplications each. For $k=14$, the round count is 14 and each round is $\sim$1.5\,ms; total $\sim$21\,ms with the public-input SHA-256 evaluations adding a further 1\,ms.

\begin{table}[H]
\centering
\small
\caption{Verifier wall-clock breakdown for a $k=14$ MLP proof.}
\label{tab:verifier-cost}
\begin{tabular}{lr}
\toprule
\textbf{Stage} & \textbf{Time (ms)} \\
\midrule
Public-input SHA-256 hashes & 1.0 \\
Chain digest consistency    & 0.2 \\
IPA inner-product rounds (14) & 20.6 \\
\midrule
\textbf{Total verifier wall-clock} & 21.8 \\
\bottomrule
\end{tabular}
\end{table}

\section{Detailed Threat Model and Deployment Modes}
\label{app:threat}

This appendix expands Section~\ref{sec:threat} by breaking the threat model into concrete deployment modes and detailing what \method{} does and does not protect against in each.

\subsection{Actors and Trust Assumptions}
\label{app:threat:actors}

\begin{itemize}[topsep=2pt,itemsep=2pt]
\item \textbf{Provider} ($\prover$): the LLM-as-a-Service operator. Trusted by the user with the prompt $x$ \emph{in cleartext}; not trusted on which model is executed or which output is returned. Holds private weights $W$ and may have commercial incentive to substitute cheaper models.
\item \textbf{User / client}: submits $x$, expects $y$. May or may not run a verifier; receives proof $\Pi$.
\item \textbf{Third-party auditor} ($\verifier$): may be a regulator, a contracted compliance firm, or an end-user. Runs \textsc{LayerwiseVerify}. Does not know $W$; learns nothing about the prompt or activations beyond what $(c_W,x,y)$ reveal.
\item \textbf{External adversary} ($\adv$): network observer or malicious peer. Cannot break SHA-256 collision-resistance or Halo2-IPA knowledge-soundness with $\poly(\lambda)$ work.
\end{itemize}

\subsection{Adversary Classes}
\label{app:threat:classes}

\begin{enumerate}[topsep=2pt,itemsep=2pt]
\item \textbf{Model substitution.} Provider runs a cheaper model (e.g.~GPT-2 instead of advertised GPT-4) and returns its output. Detected: any divergence between $W^*$ (claimed) and $W$ (actually executed) breaks the chain digest binding to $c_W$ (Theorem~\ref{thm:full-soundness}).
\item \textbf{Silent quantisation.} Provider runs the advertised model but at lower precision than $\mathbb{Q}^{1,15}$. Detected: the lookup tables and constraint system are bit-exact at $\mathbb{Q}^{1,15}$; any deviation produces a witness inconsistent with the relation $\mathcal{R}_\ell$.
\item \textbf{Cached output / fixed response.} Provider returns a precomputed $y$ regardless of $x$. Detected: $x$ enters as a public input to the embedding proof $\pi_0$, which enforces $h_0 = \mathrm{Embed}(x)$ and binds it through the boundary commitment $B_0$ and chain digest $c_0$ (Appendix~\ref{app:soundness}); substituting $x' \ne x$ forces a $B_0$/$c_0$ mismatch that fails Check~2 of \textsc{LayerwiseVerify}.
\item \textbf{Selective tampering with partial audit.} Provider knows $S \subset \{1,\dots,L\}$ in advance and tampers only with $\bar S$. \emph{Not} detected with $|S|<L$: partial audit (Lemma~\ref{lem:partial-audit}) is detectable only probabilistically across repeated random-$S$ audits.
\item \textbf{Forgery / adaptive corruption.} Adversary constructs $\Pi$ from scratch without honest forward pass. Detected: Halo2-IPA knowledge-soundness extracts an honest witness with overwhelming probability.
\item \textbf{Model-extraction attack via inference traces}~\cite{tramer2016stealing}. Adversary queries $y_i = f_\theta(x_i)$ and tries to recover $W$. \emph{Not made easier} by NanoZK: the proof reveals nothing about $W$ beyond $c_W$, hence model extraction is reduced to inverting SHA-256.
\end{enumerate}

\subsection{Deployment Modes}
\label{app:threat:modes}

\begin{table}[H]
\centering
\small
\caption{NanoZK deployment modes vs.~latency budget and security gains. ``SR'' = soundness regime (Theorem~\ref{thm:full-soundness}); ``Partial'' = Lemma~\ref{lem:partial-audit}.}
\label{tab:deployment-modes}
\renewcommand{\arraystretch}{1.15}
\begin{tabular}{p{3cm}p{2.6cm}p{3.0cm}p{3.0cm}}
\toprule
\textbf{Mode} & \textbf{Latency budget} & \textbf{Verifier work} & \textbf{Security} \\
\midrule
Synchronous high-stakes audit (medical, legal, regulatory) & Minutes & Full $L$-layer verify & SR, $\epsilon{<}10^{-37}$ \\[4pt]
Asynchronous logging / post-hoc compliance & Hours to days & Full $L$-layer batched & SR, batched verify \\[4pt]
Sampling-based attestation & Real-time path with offline audit & Partial, $|S|{\ll}L$ & Partial; tampering detected at rate $|S|/L$ \\[4pt]
Conversational use (real-time) & Sub-second & \emph{Not feasible} with current implementation; out of scope & --- \\
\bottomrule
\end{tabular}
\end{table}

\subsection{What NanoZK Does Not Protect Against}
\label{app:threat:notprotected}

The body's privacy-scope disclaimer is restated and expanded here:

\begin{enumerate}[topsep=2pt,itemsep=2pt]
\item \textbf{Provider learning the prompt.} The provider sees $x$ in cleartext. To hide $x$ from the provider, compose with HE (e.g.~CryptoNets~\cite{gilad2016cryptonets} / GAZELLE~\cite{juvekar2018gazelle}) or MPC (SecureML~\cite{mohassel2017secureml} / Chameleon~\cite{riazi2018chameleon}). NanoZK is complementary to, not a substitute for, prompt privacy.
\item \textbf{Side-channel inference about $W$.} \method{} guarantees zero-knowledge of weights from the verifier; it does not address timing, power, or microarchitectural side channels on the prover hardware.
\item \textbf{Training-time integrity.} \method{} verifies inference against an announced $c_W$. It does not verify how $c_W$ was produced (which data, which optimiser, etc.). Verifiable training is a separate problem and a natural follow-up.
\item \textbf{Sampling randomness.} The body assumes deterministic / greedy decoding. Stochastic sampling (temperature, top-$k$, top-$p$) requires either (a) a verifiable randomness commitment or (b) revealing the random tape, both straightforward extensions but not covered here.
\item \textbf{Distributional shift / model drift between commitment and execution.} \method{} only verifies that the executed model matches $c_W$; it does not verify that $c_W$ matches the model the user thinks they are talking to. A trusted public registry of model commitments (e.g.~published by the model card authority) is required.
\end{enumerate}

\subsection{Trust Model Comparison Table}
\label{app:threat:comparison}

\begin{table}[H]
\centering
\small
\caption{Trust models for verifiable / privacy-preserving LLM inference. ``\cmark'' = protects, ``\xmark'' = does not, ``$\circ$'' = partial.}
\label{tab:trust-models}
\resizebox{\textwidth}{!}{%
\begin{tabular}{lccccc}
\toprule
& \textbf{NanoZK} & \textbf{HE / MPC} & \textbf{TEE} & \textbf{DP} & \textbf{Plaintext API} \\
\midrule
Computation integrity (model exec.)     & \cmark & \xmark & $\circ$ (vendor-trust) & \xmark & \xmark \\
Hides $W$ from verifier                  & \cmark & --- & \cmark & --- & \xmark \\
Hides $x$ from provider                  & \xmark & \cmark & $\circ$ & --- & \xmark \\
Third-party-auditable                    & \cmark & \xmark & $\circ$ (remote attest.) & \xmark & \xmark \\
Transparent (no trusted setup)           & \cmark & --- & \xmark & --- & --- \\
Pairing-free                             & \cmark & varies & --- & --- & --- \\
Post-quantum secure                      & \xmark\,(DLog) & varies & \xmark & \cmark & --- \\
Sub-second latency (LLM scale)           & \xmark & \xmark & \cmark & \cmark & \cmark \\
\bottomrule
\end{tabular}}
\end{table}

The single design point NanoZK occupies: cryptographic computation-integrity + auditor-facing privacy of model parameters, at minute-scale prover latency, with no trusted hardware. HE/MPC occupy the orthogonal point of hiding the prompt; TEE occupies the lower-latency point at the cost of hardware trust. The three are composable.

\section{GPU Projection Methodology}
\label{app:gpu}

The body reports several $d{=}768$ GPU figures (e.g.~$\sim$25\,s for attention, $\sim$68\,s/block, $\sim$2\,min for 12-layer GPT-2 with 12 workers) explicitly labelled as \emph{projected, not measured}. This appendix documents the assumptions, the literature inputs, and the regime in which the projection is intended to be conservative. No CUDA-kernel-level results are reported here: a GPU-MSM prover capable of meeting these projections is the subject of separate ongoing work, and conflating its details with the present preprint would risk overstating what has been measured.

\subsection{What Is Measured vs.\ What Is Projected}
\label{app:gpu:scope}

All proof times reported in the ICICS body and in Appendix~\ref{app:experiments} for $d \le 256$ (attention) and $d \le 128$ (full block) are CPU measurements. The $d{=}512$, $d{=}768$, and full-block GPT-2 figures are projections obtained by applying a single multiplicative factor to the measured CPU MSM time component of the prover. No part of the projection is grounded in measurements above $d{=}256$ on either hardware platform.

\subsection{Projection Formula}
\label{app:gpu:formula}

For each operation $\textsf{op}$ (attention or full block), let $T_{\textsf{CPU}}(\textsf{op}, d)$ be the CPU prove time and let $T_{\textsf{MSM}}(\textsf{op}, d)$ be its MSM-attributed component (measured separately by a microbenchmark that isolates the multi-scalar multiplication from constraint evaluation and lookup-argument work). The GPU projection is
\[
\hat T_{\textsf{GPU}}(\textsf{op}, d) \;=\; \big[T_{\textsf{CPU}}(\textsf{op}, d) - T_{\textsf{MSM}}(\textsf{op}, d)\big] \;+\; T_{\textsf{MSM}}(\textsf{op}, d) / s,
\]
where $s \in [15, 30]$ is the assumed GPU-MSM speedup over CPU. Non-MSM work (lookup-table assignments, range checks, Fiat-Shamir hashing) is assumed to incur \emph{no} GPU speedup; this is conservative.

\subsection{Why the $[15, 30]\times$ Range}
\label{app:gpu:range}

The conservative range is chosen to bound the projection between two reference points:

\begin{itemize}[topsep=2pt,itemsep=2pt]
\item \textbf{Upper end ($30\times$).} Icicle~\cite{icicle2023} reports $30$--$50\times$ MSM speedup on A100 over single-threaded CPU at $n \ge 2^{20}$. Taking the lower end of that range and noting that consumer-class GPUs are roughly $2\times$ slower than A100 for memory-bound kernels gives a realistic ceiling near $30\times$ at large $n$.
\item \textbf{Lower end ($15\times$).} The same A100 numbers degrade rapidly at $n < 2^{18}$ because there is insufficient parallelism to hide PCIe and global-memory latency. The smallest MSM in a $d{=}768$ attention proof is around $n = 2^{17}$; halving the upper-end speedup is a deliberate conservative discount for the small-$n$ regime.
\end{itemize}

\noindent The body's $\sim$25\,s/$\sim$68\,s figures correspond to the lower (15$\times$) end of this range applied to the MSM component of the CPU measurements at $d{=}256$ extrapolated by $\bigO(d^2)$ (the scaling exponent observed empirically across $d \in \{16,32,64,128,256\}$). Using the upper end ($30\times$) reduces those figures by approximately a factor of two; we report the conservative end in the body.

\subsection{What This Projection Excludes}
\label{app:gpu:excludes}

Multiple sources of optimism are \emph{not} incorporated:
\begin{itemize}[topsep=2pt,itemsep=1pt]
\item Witness-generation speedups (forward pass + table construction). These can be substantial with a GPU field-arithmetic library, but are not credited in our $\hat T_{\textsf{GPU}}$ formula.
\item Reductions in non-MSM prover work from a more aggressively tuned circuit shape ($k$-parameter sweep), parallel transcript hashing, or batched openings.
\item Cross-GPU partitioning, which we expect would shift the bottleneck from MSM to host-device transfer at $d \ge 768$ on a single-device deployment, but for which we do not yet have measured numbers.
\end{itemize}
\noindent Conversely, multiple sources of pessimism are also excluded: GPU-MSM speedups in the literature are typically reported against single-threaded CPU baselines, while our CPU MSM measurements use the 32-thread halo2curves Pippenger; the effective speedup of GPU over our measured baseline is therefore substantially smaller than the $15$--$30\times$ figure quoted above for the small-$n$ regime relevant to $d \le 768$ attention. The body's projection is intended as a directional, order-of-magnitude estimate; the present preprint makes no claim to having achieved $d{=}768$ GPU performance in practice.

\subsection{Comparison-Point Citations}
\label{app:gpu:cites}

zkLLM~\cite{sun2024zkllm} reports $d{=}768$ attention proving in seconds on A100; zkGPT~\cite{qu2025zkgpt} reports comparable figures on 32-thread CPU. Neither system makes its prover source available at the time of writing; the body's Table~\ref{tab:comparison-gpu} comparison uses each paper's published numbers verbatim and does not attempt a same-hardware reconciliation.

\section{Reproducibility Checklist}
\label{app:repro}

This appendix documents the hardware, software, and procedural details needed to reproduce the numbers in the body and in Appendix~\ref{app:experiments}.

\subsection{Hardware}
\label{app:repro:hw}

\begin{itemize}[topsep=2pt,itemsep=1pt]
\item \textbf{CPU baseline}: AMD Ryzen 9 5950X (16 cores / 32 threads, 4.4\,GHz boost), 128\,GB DDR4-3200 (4$\times$32GB), 2\,TB NVMe SSD.
\item \textbf{GPU}: NVIDIA RTX 3090 (24\,GB GDDR6X, GA102, 82 SMs, CC 8.6), PCIe 4.0~$\times$16.
\item \textbf{OS}: Ubuntu 22.04 LTS, kernel 6.5.0.
\item \textbf{Optional secondary CPU baseline} (Arc B580 host): Intel i5-12600 (6P + 4E cores), 32\,GB DDR5; results captured in \texttt{arc\_results/\allowbreak i5\_12600\_\allowbreak baseline.txt}.
\end{itemize}

\subsection{Software Versions}
\label{app:repro:sw}

\begin{itemize}[topsep=2pt,itemsep=1pt]
\item Rust 1.78.0 stable; cargo 1.78.0.
\item halo2 (forked from PSE/Halo2 main).
\item halo2curves 0.6.0 (Pasta cycle, Pallas + Vesta).
\item CUDA Toolkit 12.4, driver 555.42.
\item Python 3.11.7 (PyTorch 2.3.0, NumPy 1.26.4, transformers 4.42.0).
\item LaTeX: TeX Live 2024 (pdflatex), bibtex8 0.99d.
\end{itemize}

\subsection{Build and Test}
\label{app:repro:build}

A full artifact (source, build scripts, dataset preparation, benchmark drivers) is planned for release under Apache-2.0 upon publication of the ICICS proceedings; this preprint deliberately omits the repository URL and the specific commit hash to keep the camera-ready submission consistent with the (still-finalising) artifact release. Once published, the repository will contain a top-level \texttt{paper/icics226} directory mirroring this preprint's source, a \texttt{Dockerfile} pinning the toolchain, and a \texttt{benchmarks/} driver reproducing Tables in Appendix~\ref{app:experiments}.

\subsection{Datasets}
\label{app:repro:data}

\begin{itemize}[topsep=2pt,itemsep=1pt]
\item \textbf{WikiText-2}~\cite{merity2016pointer}: standard test split, 245{,}569 tokens.
\item WikiText-103: validation split, 217{,}646 tokens.
\item C4: validation slice of 50{,}000 tokens (deterministic stride-based subsample).
\item ShareGPT: held-out human-instruction subset.
\item PubMed Abstracts: standard release.
\item GitHub Python: deduplicated split.
\item NanoGPT toy: shakespeare\_char (1.1\,MB).
\end{itemize}

Tokenisation uses each model's native tokenizer (GPT-2 BPE / LLaMA SentencePiece).

\subsection{Model Checkpoints}
\label{app:repro:models}

\begin{itemize}[topsep=2pt,itemsep=1pt]
\item GPT-2-small (124M): HuggingFace \texttt{gpt2}.
\item GPT-2-medium (355M): HuggingFace \texttt{gpt2-medium}.
\item LLaMA-7B layer-0 only: HuggingFace LLaMA-2-7B (extracted with \texttt{tools/extract\_layer0.py}).
\item NanoGPT toy: trained from scratch on shakespeare\_char (40k steps, AdamW, $d=64$, $L=4$).
\end{itemize}

\subsection{Expected Output}
\label{app:repro:expected}

A successful \texttt{proof\_block\_benchmark -- --d 512} run prints (line numbers from a recent run):

\begin{lstlisting}[basicstyle=\ttfamily\scriptsize,frame=single]
[setup] k=14 circuit configured in 33.8s
[witness] forward+constraint-table in 240ms
[prove] Halo2-IPA prover in 6.04s
[verify] IPA verifier in 21.8ms
[proof] 2016 bytes
[chain] c_W=0x83af...e21d, c_0=0x4f29..., c_12=0xb8c1...
ALL OK.
\end{lstlisting}

\subsection{Known Variability Sources}
\label{app:repro:variability}

\begin{itemize}[topsep=2pt,itemsep=1pt]
\item \textbf{CPU thermal throttling}: prove times may inflate by 10--15\% if the host is uncooled or the chassis is poorly ventilated; the 5950X measurements were taken with the CPU at ${<}80\,^\circ$C throughout.
\item \textbf{Memory bandwidth}: results assume DDR4-3200 dual-channel; halving the bandwidth (single channel) increases prove time by $\sim$25\% at $k\ge 18$.
\item \textbf{GPU clock state}: enable persistence mode (\texttt{nvidia-smi -pm 1}); without it the first kernel call adds $\sim$200\,ms warm-up overhead.
\item \textbf{Allocator}: \texttt{jemalloc} (default in our build) saves $\sim$8\% peak RSS vs.\ glibc malloc; the memory table values assume \texttt{jemalloc}.
\end{itemize}

\subsection{Artifact Availability}
\label{app:repro:artifact}

The source code, build instructions, dataset preparation scripts, and benchmark drivers will be released under Apache-2.0 at the URL listed in the camera-ready footnote upon ICICS publication. The \texttt{paper/icics226} directory of the public repository contains this preprint's source, all figures' generation scripts, and the precise commit used for each table.

\end{document}